\documentclass[default,iicol]{sn-jnl}

\usepackage{natbib}
\setcitestyle{authoryear,open={(},close={)}}
\usepackage{hyperref}
\usepackage{graphicx}
\usepackage{comment}
\usepackage{subcaption}
\usepackage{amsmath}
\usepackage{amssymb,amsthm}
\usepackage{mathtools}
\usepackage{booktabs}
\usepackage{algorithm, setspace}
\usepackage{algpseudocode, array, multirow}
\usepackage{caption}
\usepackage{xcolor}
\usepackage{extarrows}
\usepackage{enumitem}
\usepackage{rotating}
\usepackage{fixltx2e}
\usepackage{textcomp}

\definecolor{darkgreen}{rgb}{0.035, 0.412, 0.098}
\definecolor{mygreen}{rgb}{0.05, 0.71, 0.47}
\definecolor{revision}{rgb}{0.0, 0.0, 0.0}

\definecolor{R1}{RGB}{2, 106, 250}
\definecolor{R2}{RGB}{250, 184, 2}
\definecolor{R3}{RGB}{247, 40, 195}
\definecolor{AE}{RGB}{0, 168, 150}

\usepackage[capitalize]{cleveref}
\crefname{section}{Section}{Sections}
\crefname{table}{Table}{Tables}
\Crefname{equation}{Eq.}{Eqs.}
\Crefname{figure}{Figure}{Figures}

\newcommand{\norm}[1]{\left\lVert#1\right\rVert}
\renewcommand{\textuparrow}{$\uparrow$}
\renewcommand{\textdownarrow}{$\downarrow$}
\DeclareMathOperator*{\argmax}{arg\,max}
\DeclareMathOperator*{\argmin}{arg\,min}

\jyear{2022}%
\theoremstyle{thmstyleone}%
\newtheorem{theorem}{Theorem}
\newtheorem{lemma}{Lemma}
\newtheorem{innercustomthm}{Theorem}
\newenvironment{customthm}[1]
{\renewcommand\theinnercustomthm{#1}\innercustomthm}
{\endinnercustomthm}

\theoremstyle{thmstyletwo}%

\theoremstyle{thmstylethree}%

\begin{document}

\title[Adversarial Coreset Selection for Efficient Robust Training]{Adversarial Coreset Selection for Efficient Robust Training}


\author*[1]{\fnm{Hadi} \sur{M.~Dolatabadi}}\email{h.dolatabadi@unimelb.edu.au}

\author[1]{\fnm{Sarah} \sur{M.~Erfani}}\email{sarah.erfani@unimelb.edu.au}

\author[1]{\fnm{Christopher} \sur{Leckie}}\email{caleckie@unimelb.edu.au}

\affil*[1]{\orgdiv{School of Computing and Information Systems}, \orgname{The University of Melbourne}, \orgaddress{Melbourne Connect, \street{700 Swanston Street}, \city{Carlton}, \postcode{3053}, \state{Victoria}, \country{Australia}}}

\abstract{
It has been shown that neural networks are vulnerable to adversarial attacks: adding well-crafted, imperceptible perturbations to their input can modify their output.
Adversarial training is one of the most effective approaches to training robust models against such attacks.
Unfortunately, this method is much slower than vanilla training of neural networks since it needs to construct adversarial examples for the entire training data at every iteration.
By leveraging the theory of coreset selection, we show how selecting a small subset of training data provides a principled approach to reducing the time complexity of robust training.
To this end, we first provide convergence guarantees for adversarial coreset selection.
In particular, we show that the convergence bound is directly related to how well our coresets can approximate the gradient computed over the entire training data.
Motivated by our theoretical analysis, we propose using this gradient approximation error as our adversarial coreset selection objective to reduce the training set size effectively.
Once built, we run adversarial training over this subset of the training data.
Unlike existing methods, our approach can be adapted to a wide variety of training objectives, including TRADES, $\ell_p$-PGD, and Perceptual Adversarial Training.
We conduct extensive experiments to demonstrate that our approach speeds up adversarial training by 2-3 times while experiencing a slight degradation in the clean and robust accuracy.}


\maketitle
\section{Introduction}\label{sec:introduction}

Neural networks have achieved great success in the past decade.
Today, they are one of the primary candidates in solving a wide variety of machine learning tasks, from object detection and classification~\citep{he2016deep,wu2019detectron2} to photo-realistic image generation~\citep{karras2020stylegan2,vahdat2020nvae} and beyond.
Despite their impressive performance, neural networks are vulnerable to adversarial attacks~\citep{biggio2013evasion,szegedy2014intriguing}: adding well-crafted, imperceptible perturbations to their input can change their output.
This unexpected behavior of neural networks prevents their widespread deployment in safety-critical applications, including autonomous driving~\citep{eykholt2018robust} and medical diagnosis~\citep{ma2021understanding}.
As such, training robust neural networks against adversarial attacks is of paramount importance and has gained ample attention.

\textit{Adversarial training} is one of the most successful approaches in defending neural networks against adversarial attacks.
This approach first constructs a perturbed version of the training data.
Then, the neural network is optimized over these perturbed inputs instead of the clean samples.
This procedure must be done iteratively as the perturbations depend on the neural network weights.
Since the weights are optimized during training, the perturbations also need to be adjusted for each data sample in every iteration.\footnote{Note that adversarial training in the literature generally refers to a particular approach proposed by~\citet{madry2018towards}.
In this paper, we refer to any method that builds adversarial attacks around the training data and incorporates them into the training of the neural network as adversarial training. Using this taxonomy, methods such as TRADES~\citep{zhang2019trades}, $\ell_p$-PGD~\citep{madry2018towards} or Perceptual Adversarial Training~(PAT)~\citep{laidlaw2021pat} are all considered different versions of adversarial training.}

Various adversarial training methods primarily differ in how they define and find the perturbed version of the input~\citep{madry2018towards,zhang2019trades,laidlaw2021pat}.
However, they all require repetitive construction of these perturbations during training which is often cast as another non-linear optimization problem.
Therefore, the time/computational complexity of adversarial training is much higher than vanilla training.
In practice, neural networks require massive amounts of training data~\citep{adadi2021data} and need to be trained multiple times with various hyper-parameters to get their best performance~\citep{killamsetty2021gradmatch}.
Thus, reducing the time/computational complexity of adversarial training is critical to enabling the environmentally efficient application of robust neural networks in real-world scenarios~\citep{schwartz2020greenai,strubell2019energy}.

\textit{Fast Adversarial Training}~(FAT)~\citep{wong2020fast} is a successful approach proposed for efficient training of robust neural networks.
Contrary to the common belief that building the perturbed versions of the inputs using \textit{Fast Gradient Sign Method}~(FGSM)~\citep{goodfellow2014explaining} does not help in training arbitrary robust models~\citep{tramer2018ensemble,madry2018towards}, \citet{wong2020fast} show that by carefully applying uniformly random initialization before the FGSM step one can make this training approach work.
Using FGSM to generate the perturbed input in a single step combined with implementation tricks such as mixed precision and cyclic learning rate~\citep{smith2017cyclic}, FAT can significantly reduce the training time of robust neural networks.

Despite its success, FAT may exhibit unexpected behaviors in different settings.
For instance, it was shown that FAT suffers from \textit{catastrophic overfitting} where the robust accuracy during training suddenly drops to 0\%~\citep{wong2020fast,andriushchenko2020understanding}.
\textcolor{revision}{Another fundamental issue with FAT and its variations such as \texttt{GradAlign}~\citep{andriushchenko2020understanding} and \textit{N-FGSM}~\citep{aranda2022nfgsm} is that they are specifically designed and implemented for $\ell_\infty$ adversarial training.}
This is because FGSM, particularly an $\ell_\infty$ perturbation generator, is at the heart of these methods.
As a result, the quest for a unified approach that can reduce the time complexity of all types of adversarial training is not over.

Motivated by the limited scope of FAT, in this paper, we take an important step toward finding a general yet principled approach for reducing the time complexity of adversarial training.
We notice that repetitive construction of adversarial examples for each data point is the main bottleneck of robust training.
While this needs to be done iteratively, we speculate that perhaps we can find a subset of the training data that is more important to robust network optimization than the rest.
Specifically, we ask the following research question: \textit{Can we train an adversarially robust neural network using a subset of the entire training data without sacrificing clean or robust accuracy?}

\begin{figure*}[t!]
	\centering
	\begin{subfigure}{\textwidth}
		\centering
		\includegraphics[width=1.0\textwidth]{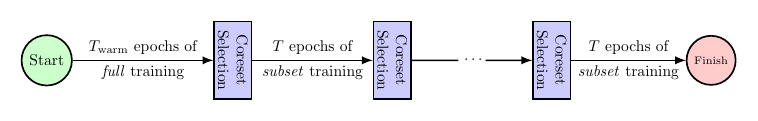}
		\caption{Selection is done every $T$ epochs. During the next episodes, the network is only trained on this subset.}
	\end{subfigure}\\\vspace{1em}
	\begin{subfigure}{0.44\textwidth}
		\centering
        \includegraphics[width=0.75\textwidth]{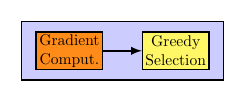}
		\caption{\label{fig:coresets_idea_van} Coreset selection module for vanilla training.}
	\end{subfigure}\hspace{1em}
	\begin{subfigure}{0.48\textwidth}
		\centering
		\includegraphics[width=1.0\textwidth]{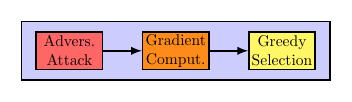}
		\caption{\label{fig:coresets_idea_adv} Coreset selection module for adversarial training.}
	\end{subfigure}
	\caption{\label{fig:coresets_idea} Overview of neural network training using coreset selection. Contrary to vanilla coreset selection, in our adversarial version we first need to construct adversarial examples and then perform coreset selection.}
\end{figure*}

This paper shows that the answer to this question is affirmative:
by selecting a \textit{weighted} subset of the data based on the neural network state, we run \textit{weighted} adversarial training only on this selected subset.
To achieve this goal, we first theoretically analyze adversarial subset selection convergence under gradient descent for a few idealistic settings.
Our study demonstrates that the convergence bound is directly related to the capability of the \textit{weighted} subset in approximating the loss gradient over the entire training set.
Motivated by this analysis, we propose using the gradient approximation error as our adversarial coreset selection objective for training robust neural networks.
We then draw an elegant connection between adversarial training and vanilla coreset selection algorithms.
In particular, we use Danskin's theorem and demonstrate how the entire training data can effectively be approximated with an informative weighted subset.
To conduct this selection, our study shows that one needs to build adversarial examples for the entire training data and solve a respective subset selection objective.
Afterward, training can be performed on this selected subset of the training data.
In our approach, shown in \Cref{fig:coresets_idea}, adversarial coreset selection is only required every few epochs, effectively reducing the time complexity of robust training algorithms.
We demonstrate how our proposed approach can be used as a general framework in conjunction with different adversarial training objectives, opening the door to a more principled approach for efficient training of robust neural networks in a general setting.
Our experimental results show that one can reduce the time complexity of various robust training objectives by 2-3 times without sacrificing too much clean and robust accuracy.

In summary, we make the following contributions:
\begin{itemize}\setlength\itemsep{0.25em}
	\item We propose a practical yet principled algorithm for efficient training of robust neural networks based on adaptive coreset selection. To the best of our knowledge, we are the first to use coreset selection in adversarial training.
    \item We provide theoretical guarantees for the convergence of our adversarial coreset selection algorithm under different settings.
	\item Based on our theoretical study, we develop adversarial coreset selection for neural networks and show that our approach can be applied to a variety of robust learning objectives, including TRADES~\citep{zhang2019trades}, $\ell_p$-PGD~\citep{madry2018towards} and Perceptual~\citep{laidlaw2021pat} Adversarial Training. Our approach encompasses a broader range of robust training objectives compared to the limited scope of the existing methods.
	\item Our experiments demonstrate that the proposed approach can result in a 2-3 fold reduction of the training time in adversarial training, with only a slight reduction in the clean and robust accuracy.
\end{itemize}

The rest of this paper is organized as follows.
In \Cref{sec:background}, we go over the preliminaries of our work and review the related work.
We then propose our approach in \Cref{sec:proposed_method}.
Next, we present and discuss our experimental results in \Cref{sec:experiments}.
Finally, we conclude the paper in \Cref{sec:conclusion}.

\section{Preliminaries}\label{sec:background}

In this section, we review the related background to our work.

\subsection{Adversarial Training}\label{sec:sec:adversarial_training}
Let $\mathcal{D}=\left\{\left(\boldsymbol{x}_{i}, y_{i}\right)\right\}_{i=1}^{n} \subset  \mathcal{X} \times \mathcal{Y}$ denote a training dataset consisting of $n$ i.i.d.~samples.
Each data point contains an input data $\boldsymbol{x}_{i}$ from domain $\mathcal{X}$ and an associated label $y_{i}$ taking one of $k$ possible values ${\mathcal{Y}=\left[k\right]=\left\{1, 2, \dots, k\right\}}$.
Without loss of generality, in this paper we focus on the image domain $\mathcal{X}$.
Furthermore, assume that ${f_{\boldsymbol{\theta}}: \mathcal{X} \rightarrow \mathbb{R}^{k}}$ denotes a neural network classifier with parameters $\boldsymbol{\theta}$ that takes $\boldsymbol{x} \in \mathcal{X}$ as input and maps it to a logit value $f_{\boldsymbol{\theta}}(\boldsymbol{x}) \in \mathbb{R}^{k}$.
Then, training a neural network in its most general format can be written as the following minimization problem:
\begin{equation}\label{eq:nn_training_objective}
\min_{\boldsymbol{\theta}} \sum_{i \in \mathcal{V}} \boldsymbol{\Phi} \left(f_{\boldsymbol{\theta}}; \boldsymbol{x}_{i}, y_{i}\right),
\end{equation}
Here, $\boldsymbol{\Phi} \left(f_{\boldsymbol{\theta}}; \boldsymbol{x}, y\right)$ is a function that takes a data point $\left(\boldsymbol{x}, y\right)$ and a function $f_{\boldsymbol{\theta}}$ as its inputs, and its output is a measure of discrepancy between the input $\boldsymbol{x}$ and its ground-truth label~$y$.
Also, $\mathcal{V}=\left[n\right]=\left\{1, 2, \dots, n\right\}$ denotes the entire training data indices.
By writing the training objective in this format, we can denote both vanilla and robust adversarial training using the same notation.
Below we show how various choices of the function $\boldsymbol{\Phi}$ amount to different training objectives.

\subsubsection{Vanilla Training}
In case of vanilla training, the function $\boldsymbol{\Phi}$ is a simple evaluation of an appropriate loss function over the neural network output $f_{\boldsymbol{\theta}}(\boldsymbol{x})$ and the ground-truth label $y$.
In other words, for vanilla training we have:
\begin{equation}\label{eq:vanilla_functional}
\boldsymbol{\Phi} \left(f_{\boldsymbol{\theta}}; \boldsymbol{x}, y\right) = \mathcal{L}_{\mathrm{CE}}\left(f_{\boldsymbol{\theta}}(\boldsymbol{x}), y\right),
\end{equation}
where $\mathcal{L}_{\mathrm{CE}}(\cdot, \cdot)$ is the cross-entropy loss.

\subsubsection{FGSM, $\ell_p$-PGD, and Perceptual Adversarial Training}
In adversarial training, the objective is itself an optimization problem:
\begin{equation}\label{eq:at_functional}
\boldsymbol{\Phi} \left(f_{\boldsymbol{\theta}}; \boldsymbol{x}, y\right) = \max_{\tilde{\boldsymbol{x}}} \mathcal{L}_{\mathrm{CE}}\left(f_{\boldsymbol{\theta}}(\tilde{\boldsymbol{x}}), y\right)~ \text{s.t.}~\mathrm{d}\left({\tilde{\boldsymbol{x}}, \boldsymbol{x}}\right)\leq \varepsilon
\end{equation}
where $\mathrm{d}(\cdot, \cdot)$~is an appropriate distance measure over image domain~$\mathcal{X}$, and $\varepsilon$ denotes a scalar.
The constraint over $\mathrm{d}\left({\tilde{\boldsymbol{x}}, \boldsymbol{x}}\right)$ is used to ensure visual similarity between $\tilde{\boldsymbol{x}}$ and $\boldsymbol{x}$.
It can be shown that solving \Cref{eq:at_functional} amounts to finding an adversarial example $\tilde{\boldsymbol{x}}$ for the clean sample $\boldsymbol{x}$~\citep{madry2018towards}.
Different choices of the visual similarity measure $\mathrm{d}(\cdot, \cdot)$ and solvers for \Cref{eq:at_functional} result in different adversarial training objectives:
\begin{itemize}
	\item FGSM~\citep{goodfellow2014explaining} assumes that ${\mathrm{d}({\tilde{\boldsymbol{x}}, \boldsymbol{x}}) = \norm{{\tilde{\boldsymbol{x}}-\boldsymbol{x}}}_{\infty}}$.
	Using this $\ell_\infty$ assumption, the solution to \Cref{eq:at_functional} is computed using one iteration of gradient ascent.
	\item $\ell_p$-PGD~\citep{madry2018towards} utilizes $\ell_p$ norms as a proxy for visual similarity $\mathrm{d}(\cdot, \cdot)$. Then, several steps of projected gradient ascent is taken to solve \Cref{eq:at_functional}.
	\item Finally, Perceptual Adversarial Training~(PAT)~\citep{laidlaw2021pat} uses \textit{Learned Perceptual Image Patch Similarity}~(LPIPS)~\citep{zhang2018lpips} as its distance measure.
	\citet{laidlaw2021pat} propose to solve the inner maximization of~\Cref{eq:at_functional} objective using either projected gradient ascent or Lagrangian relaxation.
\end{itemize}

\subsubsection{TRADES Adversarial Training}
This approach uses a combination of \Cref{eq:at_functional,eq:vanilla_functional}.
The intuition behind TRADES~\citep{zhang2019trades} is creating a trade-off between clean and robust accuracy.
In particular, the objective is written as:
\begin{align}\label{eq:trades_functional}\nonumber
	\boldsymbol{\Phi} \left(\boldsymbol{x}, y; f_{\boldsymbol{\theta}}\right) =~ &\mathcal{L}_{\mathrm{CE}}\left(f_{\boldsymbol{\theta}}(\boldsymbol{x}), y\right) \\
	&+ \max_{\tilde{\boldsymbol{x}}} \mathcal{L}_{\mathrm{CE}}\left(f_{\boldsymbol{\theta}}(\tilde{\boldsymbol{x}}), f_{\boldsymbol{\theta}}(\boldsymbol{x})\right)/\lambda,
\end{align}
such that $\mathrm{d}=({\tilde{\boldsymbol{x}}, \boldsymbol{x}})\leq \varepsilon$.
Here, $\lambda$ is a regularization parameter that controls the trade-off.

\subsection{Coreset Selection}\label{sec:sec:coreset_selection}

\begin{figure*}[tb!]
    \centering
    \includegraphics[width=0.8\textwidth]{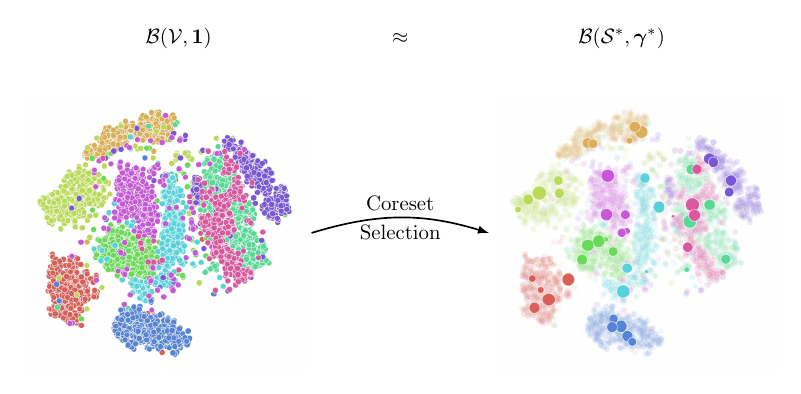}
    \caption{Coreset selection aims at finding a weighted subset of the data that can approximate certain behaviors of the entire data samples. In this figure, we denote the behavior of interest as a function $\mathcal{B}(\cdot, \cdot)$ that receives a set and its associated weights. The goal of coreset selection is to move from the original data $\mathcal{V}$ with uniform wights $\boldsymbol{1}$ to a weighted subset $\mathcal{S}^{*}$ with weights $\boldsymbol{\gamma}^{*}$.}
    \label{fig:coreset_selection}
\end{figure*}

\textit{Coreset selection}~(also referred to as adaptive data subset selection) attempts to find a weighted subset of the data that can approximate specific attributes of the entire population~\citep{feldman2011coresets}.
Coreset selection algorithms start with defining a criterion based on which the subset of interest is found:
\begin{equation}\label{eq:coreset_selection}
    \mathcal{S}^{*}, \boldsymbol{\gamma}^{*} = \argmin_{\mathcal{S} \subseteq \mathcal{V}, \boldsymbol{\gamma}} \mathcal{C}(\mathcal{S}, \boldsymbol{\gamma}).
\end{equation}
In this definition, $\mathcal{S}$ is a subset of the entire data $\mathcal{V}$, and $\boldsymbol{\gamma}$ denotes the weights associated with each sample in the subset $\mathcal{S}$.
Moreover, $\mathcal{C}(\cdot, \cdot)$ denotes a selection criterion based on which the coreset $\mathcal{S}^{*}$ and its weights $\boldsymbol{\gamma}^{*}$ are aimed to be found.
Once the coreset is found, one can work with these samples to represent the entire dataset.
\Cref{fig:coreset_selection} depicts this definition of coreset selection.

Traditionally, coreset selection has been used for different machine learning tasks such as $k$-means and $k$-medians~\citep{harpeled2004oncoresets}, Na\"{i}ve Bayes and nearest neighbor classifiers~\citep{wei2015nnc}, and Bayesian inference~\citep{campbell2018coreset}.
Recently, coreset selection algorithms are developed for neural network training~\citep{mirzasoleiman2020craig,mirzasoleiman2020crust,killamsetty2021glister,killamsetty2021gradmatch,killamsetty2021retrieve}.
The main idea behind such methods is often to approximate the full gradient using a subset of the training data.

Existing coreset selection algorithms can only be used for the vanilla training of neural networks.
As such, they still suffer from adversarial vulnerability.
This paper extends coreset selection algorithms to robust neural network training and shows how they can be adapted to various robust training objectives.

\section{Proposed Method}\label{sec:proposed_method}

The main bottleneck in the time/computational complexity of adversarial training stems from constructing adversarial examples for the entire training set at each epoch.
FAT~\citep{wong2020fast} tries to eliminate this issue by using FGSM as its adversarial example generator.
However, this simplification (1) may lead to catastrophic overfitting~\citep{wong2020fast,andriushchenko2020understanding}, and (2) is not easily applicable to different types of adversarial training as FGSM is designed explicitly for $\ell_\infty$ attacks.

Instead of using a faster adversarial example generator, here, we take a different, \textit{orthogonal} path and try to reduce the training set size effectively.
This way, the original adversarial training algorithm can still be used on this smaller subset of training data. 
This approach can reduce the time/computational complexity, while optimizing a similar objective as the initial training.
In this sense, it leads to \textit{unified} method that can be used along various types of adversarial training objectives, including the ones that already exist and the ones that will be proposed in the future. 

The main hurdle in materializing this idea is the following question:

\begin{center}
\textit{How should we select this subset of the training data while minimizing the impact on the clean or robust accuracy?}
\end{center}

To answer this question, we next provide convergence guarantees for adversarial training using a subset of the training data.
This analysis would lay the foundation of our adversarial coreset selection objective in the subsequent sections.

\subsection{Convergence Guarantees}\label{sec:sec:convergence}

This section provides theoretical insights into our proposed adversarial coreset selection.
Specifically, we aim to find a convergence bound for adversarial training over a subset of the data and see how it relates to the optimal solution.

Let $L(\boldsymbol{\theta})$ denote the adversarial training objective over the entire training dataset such that:\footnote{Note that while we investigate a usual adversarial training objective, our analysis can also be easily extended to other objectives such as TRADES.}
\begin{equation}\label{eq:emp_loss}
    L(\boldsymbol{\theta}) = \sum_{i \in \mathcal{V}} \max_{\tilde{\boldsymbol{x}}_i} \mathcal{L}(\boldsymbol{\theta}; \tilde{\boldsymbol{x}}_i),
\end{equation}
where $\mathcal{L}(\boldsymbol{\theta}; \tilde{\boldsymbol{x}}_i)$ is the evaluation of the loss over input $\tilde{\boldsymbol{x}}_i$ with network parameters $\boldsymbol{\theta}$.\footnote{Note that while the loss is also a function of the ground-truth label $y_i$, we omit it for brevity.}
The goal is to find the optimal set of parameters $\boldsymbol{\theta}$ such that this objective is minimized.
To optimize the parameters $\boldsymbol{\theta}$ of the underlying learning algorithm, we use gradient descent.
Let ${t = 0, 1, \dots, T-1}$ denote the current epoch.
Then, gradient descent update can be written as:
\begin{equation}\label{eq:gradient_descent}
    \boldsymbol{\theta}_{t + 1} = \boldsymbol{\theta}_{t} - \alpha_{t} \nabla_{\boldsymbol{\theta}}L(\boldsymbol{\theta}_{t}),
\end{equation}
where $\alpha_{t}$ is the learning rate.

As demonstrated in \Cref{eq:coreset_selection}, the ultimate goal of coreset selection is to find a subset  $\mathcal{S} \subseteq \mathcal{V}$ of the training data with weights $\boldsymbol{\gamma}$ to approximate certain behaviors of the entire population $\mathcal{V}$.
In our case, the aim is to successfully train a robust neural network over this weighted dataset using:
\begin{equation}\label{eq:emp_loss_coreset}
    L^{\mathcal{S}}_{\boldsymbol{\gamma}}{(\boldsymbol{\theta})} = \sum_{j \in \mathcal{S}} \gamma_{j} \max_{\tilde{\boldsymbol{x}}_j} \mathcal{L}(\boldsymbol{\theta}; \tilde{\boldsymbol{x}}_j)
\end{equation}
which is the weighted loss over the coreset $\mathcal{S}$.\footnote{We will sometimes denote $L^{\mathcal{S}}_{\boldsymbol{\gamma}}{(\boldsymbol{\theta})}$ as $L_{\boldsymbol{\gamma}}{(\boldsymbol{\theta})}$ for brevity.}
Once a coreset $\mathcal{S}$ is found, we can replace the gradient descent update rule in \Cref{eq:gradient_descent} with:
\begin{equation}\label{eq:gradient_descent_coreset}
    \boldsymbol{\theta}_{t + 1} \approx \boldsymbol{\theta}_{t} - \alpha_{t} \nabla_{\boldsymbol{\theta}}L_{\boldsymbol{\gamma}^{t}}(\boldsymbol{\theta}_{t}),
\end{equation}
where
\begin{equation}\label{eq:emp_loss_coreset_rep}
    L_{\boldsymbol{\gamma}^{t}}{(\boldsymbol{\theta}_{t})} = \sum_{j \in \mathcal{S}^{t}} \gamma^{t}_{j} \max_{\tilde{\boldsymbol{x}}_j} \mathcal{L}(\boldsymbol{\theta}_{t}; \tilde{\boldsymbol{x}}_j)
\end{equation}
is the weighted empirical loss over the coreset $\mathcal{S}^{t}$ at iteration~$t$.

The following theorem extends the convergence guarantees of \citet{killamsetty2021gradmatch} to adversarial training.
\begin{theorem}\label{thm:convergence}
Let $\boldsymbol{\gamma}^{t}$ and $\mathcal{S}^{t}$ denote the weights and subset derived by any \textbf{adversarial coreset selection} algorithm at iteration $t$ of the full gradient descent.
Also, let $\boldsymbol{\theta}^{*}$ be the optimal model parameters, $\mathcal{L}$ be a convex loss function with respect to $\boldsymbol{\theta}$, and that the parameters are bounded such that $\norm{\boldsymbol{\theta} - \boldsymbol{\theta}^{*}} \leq \Delta$.
Moreover, let us define the gradient approximation error at iteration $t$ with:
\begin{equation}\nonumber
    \Gamma(L, L_{\boldsymbol{\gamma}}, \boldsymbol{\gamma}^{t}, \mathcal{S}^{t}, \boldsymbol{\theta}_t) \coloneqq \norm{\nabla_{\boldsymbol{\theta}}L(\boldsymbol{\theta}_t) - \nabla_{\boldsymbol{\theta}}L^{\mathcal{S}^{t}}_{\boldsymbol{\gamma}^{t}}{(\boldsymbol{\theta}_t)}}.
\end{equation}

Then, for $t=0, 1, \cdots, T-1$ the following guarantees hold:

(1) For a Lipschitz continuous loss function $\mathcal{L}$ with parameter $\sigma$ and constant learning rate ${\alpha=\frac{\Delta}{\sigma \sqrt{T}}}$ we have:
\begin{align}\nonumber
    \min _{t=0: T-1} L(\boldsymbol{\theta}_{t})&-L(\boldsymbol{\theta}^{*}) \leq \frac{\Delta \sigma}{\sqrt{T}} \cdots \\ \nonumber & \cdots+\frac{\Delta}{T} \sum_{t=0}^{T-1} \Gamma(L, L_{\boldsymbol{\gamma}^t}, \boldsymbol{\gamma}^{t}, \mathcal{S}^{t}, \boldsymbol{\theta}_t).
\end{align}

(2) Moreover, for a Lipschitz continuous loss $\mathcal{L}$ with parameter $\sigma$ and strongly convex with parameter $\mu$, by setting a learning rate $\alpha_{t}=\frac{2}{n\mu(1+t)}$ we have:
\begin{align}\nonumber
\min _{t=0: T-1} L(\boldsymbol{\theta}_{t})&-L(\boldsymbol{\theta}^{*}) \leq \frac{2 \sigma^{2}}{n\mu(T-1)} \cdots \\ \nonumber &\cdots+\sum_{t=0}^{T-1} \frac{2 \Delta t}{T(T-1)} \Gamma(L, L_{\boldsymbol{\gamma}^t}, \boldsymbol{\gamma}^{t}, \mathcal{S}^{t}, \boldsymbol{\theta}_t),
\end{align}
where $n$ is the total number of training data.
\end{theorem}
\begin{proof}[Proof Sketch]
    We first draw a connection between the Lipschitz and strongly convex properties of the loss function $\mathcal{L}$ and its max function $\max \mathcal{L}$.
    Then, we exploit these lemmas as well as Danskin's theorem~(\Cref{danskin_theorem}) to provide the convergence guarantees.
    For more details, please see~Appendix~\ref{ap:proofs}.
\end{proof}

\subsection{Coreset Selection for Efficient Adversarial Training}\label{sec:sec:coreset_adversarial}
As our analysis in \Cref{thm:convergence} indicates, the convergence bound consists of two terms: an irreducible noise term and an additional term consisting of gradient approximation errors.
Motivated by our analysis for this idealistic setting, we set our adversarial coreset selection objective to minimize the gradient approximation error.

In particular, let us assume that we have a neural network that we aim to robustly train using:
\begin{equation}\label{eq:nn_training_objective_repeat}
\min_{\boldsymbol{\theta}} \sum_{i \in \mathcal{V}} \boldsymbol{\Phi} \left(f_{\boldsymbol{\theta}}; \boldsymbol{x}_{i}, y_{i}\right),
\end{equation}
where $\mathcal{V}$ denotes the entire training data, and $\boldsymbol{\Phi}(\cdot)$ takes one of \Cref{eq:at_functional,eq:trades_functional} formats.
We saw that we need a subset of data that can minimize the gradient approximation error to have a tight convergence bound.
This choice also makes intuitive sense: since the gradient contains the relevant information for training a neural network using gradient descent, we must attempt to find a subset of the data that can approximate the full gradient.
As such, we set the adversarial coreset selection criterion to:
\begin{align}\label{eq:gradient_app_coreset}\nonumber
	\mathcal{S}^{*}, \boldsymbol{\gamma}^{*} = \argmin_{{\mathcal{S}} \subseteq {\mathcal{V}}, {\boldsymbol{\gamma}}} \Big\lVert&\sum_{i \in {\mathcal{V}}} \nabla_{\boldsymbol{\theta}}\boldsymbol{\Phi} \left(f_{\boldsymbol{\theta}}; \boldsymbol{x}_{i}, y_{i}\right) \cdots \\ \cdots &- \sum_{j \in {\mathcal{S}}} {\gamma_{j}}\nabla_{\boldsymbol{\theta}}\boldsymbol{\Phi} \left(f_{\boldsymbol{\theta}}; \boldsymbol{x}_{j}, y_{j}\right)\Big\rVert,
\end{align}
\noindent where $\mathcal{S}^{*} \subseteq \mathcal{V}$ is the coreset, and $\gamma^{*}_{j}$'s are the weights of each sample in the coreset.
Once the coreset is found, instead of training the neural network using \Cref{eq:nn_training_objective_repeat}, we can optimize it just over the coreset using a weighted training objective
\begin{equation}\label{eq:nn_training_objective_coreset}
\min_{\boldsymbol{\theta}} \sum_{j \in \mathcal{S}^{*}} \gamma^{*}_{j} \boldsymbol{\Phi} \left(f_{\boldsymbol{\theta}}; \boldsymbol{x}_{j}, y_{j}\right).
\end{equation}

It can be shown that solving \Cref{eq:gradient_app_coreset} is NP-hard~\citep{mirzasoleiman2020craig,mirzasoleiman2020crust}.
Roughly, various coreset selection methods differ in how they approximate the solution of the aforementioned objective.
For instance, \textsc{Craig}~\citep{mirzasoleiman2020craig} casts this objective as a \textit{submodular set cover problem} and uses existing greedy solvers to get an approximate solution.
As another example, \textsc{GradMatch}~\citep{killamsetty2021gradmatch} analyzes the convergence of stochastic gradient descent using adaptive data subset selection.
Based on this study, \citet{killamsetty2021gradmatch} propose to use Orthogonal Matching Pursuit~(OMP)~\citep{pati1992omp,elenberg2016restricted} as a greedy solver of the data selection objective.
More information about these methods is provided in~Appendix~\ref{ap:sec:greedy_selection}.

The issue with the aforementioned coreset selection methods is that they are designed explicitly for vanilla training of neural networks~(see \Cref{fig:coresets_idea_van}), and they do not reflect the requirements of adversarial training.
As such, we should modify these methods to make them suitable for our purpose of robust neural network training.
Meanwhile, we should also consider the fact that the field of coreset selection is still evolving.
Thus, we aim to find a general modification that can later be used alongside newer versions of greedy coreset selection algorithms.

We notice that various coreset selection methods proposed for vanilla neural network training only differ in their choice of greedy solvers.
Therefore, we narrow down the changes we want to make to the first step of coreset selection: gradient computation.
Then, existing greedy solvers can be used to find the subset of training data that we are looking for.
To this end, we draw a connection between coreset selection methods and adversarial training using Danskin's theorem, as outlined next.
Our analysis shows that for adversarial coreset selection, one needs to add a pre-processing step where adversarial attacks for the raw training data need to be computed~(see~\Cref{fig:coresets_idea_adv}).

\subsection{From Vanilla to Adversarial Coreset Selection}\label{sec:sec:coreset_adversarial_from}

To construct the coreset selection objective given in~\Cref{eq:gradient_app_coreset}, we need to compute the loss gradient with respect to the neural network weights.
Once done, we can use existing greedy solvers to find the solution.
The gradient computation needs to be performed for the entire training set.
In particular, using our notation from \Cref{sec:sec:adversarial_training}, this step can be written as:
\begin{equation}\label{eq:nn_gradient}
\nabla_{\boldsymbol{\theta}}\boldsymbol{\Phi} \left(f_{\boldsymbol{\theta}}; \boldsymbol{x}_{i}, y_{i}\right) \quad \forall\quad i \in \mathcal{V},	
\end{equation}
where $\mathcal{V}$ denotes the training set.

For vanilla neural network training~(see \Cref{sec:sec:adversarial_training}) the above gradient is simply equal to $\nabla_{\boldsymbol{\theta}}\mathcal{L}_{\mathrm{CE}}\left(f_{\boldsymbol{\theta}}(\boldsymbol{x}_{i}), y_{i}\right)$ which can be computed using standard backpropagation.
In contrast, for the adversarial training objectives in \Cref{eq:at_functional,eq:trades_functional}, this gradient requires taking partial derivative of a maximization objective.
To this end, we use the famous Danskin's theorem~\citep{danskin1967theory} as stated below.

\begin{theorem}[Theorem A.1 in \citet{madry2018towards}]\label{danskin_theorem}
	Let $\mathcal{K}$ be a nonempty compact topological space, ${\mathcal{L}: \mathbb{R}^{m} \times \mathcal{K} \rightarrow \mathbb{R}}$ be such that $\mathcal{L}(\cdot, \boldsymbol{\delta})$ is differentiable and convex for every $\boldsymbol{\delta} \in \mathcal{K}$, and $\nabla_{\boldsymbol{\theta}} \mathcal{L}(\boldsymbol{\theta}, \boldsymbol{\delta})$ is continuous on $\mathbb{R}^{m} \times \mathcal{K}$.
	Also, let ${\boldsymbol{\delta}^{*}(\boldsymbol{\theta})=\left\{\boldsymbol{\delta} \in \arg \max _{\boldsymbol{\delta} \in \mathcal{K}} \mathcal{L}(\boldsymbol{\theta}, \boldsymbol{\delta})\right\}}$.
	Then, the corresponding max-function
	$$
	\phi(\boldsymbol{\theta})=\max _{\delta \in \mathcal{K}} \mathcal{L}(\boldsymbol{\theta}, \boldsymbol{\delta})
	$$
	is locally Lipschitz continuous, convex, directionally differentiable, and its directional derivatives along vector $\boldsymbol{h}$ satisfy
	$$
	\phi^{\prime}(\boldsymbol{\theta}, \boldsymbol{h})=\sup _{\boldsymbol{\delta} \in \boldsymbol{\delta}^{*}(\boldsymbol{\theta})} \boldsymbol{h}^{\top} \nabla_{\boldsymbol{\theta}} \mathcal{L}(\boldsymbol{\theta}, \boldsymbol{\delta}).
	$$
	In particular, if for some $\boldsymbol{\theta} \in \mathbb{R}^{m}$ the set $\boldsymbol{\delta}^{*}(\boldsymbol{\theta})=\left\{\boldsymbol{\delta}_{\boldsymbol{\theta}}^{*}\right\}$ is a singleton, then the max-function is differentiable at $\boldsymbol{\theta}$ and
	$$
	\nabla \phi(\boldsymbol{\theta})=\nabla_{\boldsymbol{\theta}} \mathcal{L}\left(\boldsymbol{\theta}, \boldsymbol{\delta}_{\boldsymbol{\theta}}^{*}\right).
	$$
\end{theorem}

In summary, \Cref{danskin_theorem} indicates how to take the gradient of a max-function.
To this end, it suffices to (1) find the maximizer, and (2) evaluate the normal gradient at this point.

Now that we have stated Danskin's theorem, we are ready to show how it can provide the connection between vanilla coreset selection and the adversarial training objectives of \Cref{eq:at_functional,eq:trades_functional}.
We show this for the two cases of adversarial training and TRADES, but it can also be used for any other robust training objective.

\subsubsection{\textcolor{blue}{Case 1} ($\ell_p$-PGD and Perceptual Adversarial Training)}
Going back to \Cref{eq:nn_gradient}, we need to compute this gradient term for our coreset selection objective in \Cref{eq:at_functional}.
In particular, we need to compute:
\begin{equation}\label{eq:nn_gradient_at}
\nabla_{\boldsymbol{\theta}}\boldsymbol{\Phi} \left(f_{\boldsymbol{\theta}}; \boldsymbol{x}, y\right) =  \nabla_{\boldsymbol{\theta}} \max_{\tilde{\boldsymbol{x}}} \mathcal{L}_{\mathrm{CE}}\left(f_{\boldsymbol{\theta}}(\tilde{\boldsymbol{x}}), y\right)
\end{equation}
under the constraint $\mathrm{d}({\tilde{\boldsymbol{x}}, \boldsymbol{x}})\leq \varepsilon$ for every training sample.
Based on Danskin's theorem, we deduce:
\begin{equation}\label{eq:nn_gradient_at_penufinal}
\nabla_{\boldsymbol{\theta}}\boldsymbol{\Phi} \left(f_{\boldsymbol{\theta}}; \boldsymbol{x}, y\right) =  \nabla_{\boldsymbol{\theta}} \mathcal{L}_{\mathrm{CE}}\left(f_{\boldsymbol{\theta}}({\boldsymbol{x}^{*}}), y\right),
\end{equation}
where $\boldsymbol{x}^{*}$ is the solution to:
\begin{equation}\label{eq:at_max_objective}
\max_{\tilde{\boldsymbol{x}}} \mathcal{L}_{\mathrm{CE}}\left(f_{\boldsymbol{\theta}}(\tilde{\boldsymbol{x}}), y\right) \quad \text{s.t.} \quad \mathrm{d}({\tilde{\boldsymbol{x}}, \boldsymbol{x}})\leq \varepsilon.
\end{equation}
The conditions under which Danskin's theorem hold might not be satisfied for neural networks in general.
This is due to the presence of functions with discontinuous gradients, such as ReLU activation, in neural networks.
More importantly, finding the exact solution of \Cref{eq:at_max_objective} is not straightforward as neural networks are highly non-convex.
Usually, the exact solution $\boldsymbol{x}^{*}$ is replaced with its approximation, which is an adversarial example generated under the \Cref{eq:at_max_objective} objective~\citep{madry2018adversarial}.
Based on this approximation, we can re-write \Cref{eq:nn_gradient_at_penufinal} as:
\begin{equation}\label{eq:nn_gradient_at_final}
\nabla_{\boldsymbol{\theta}}\boldsymbol{\Phi} \left(f_{\boldsymbol{\theta}}; \boldsymbol{x}, y\right) \approx  \nabla_{\boldsymbol{\theta}} \mathcal{L}_{\mathrm{CE}}\left(f_{\boldsymbol{\theta}}({\boldsymbol{x}_{\rm{adv}}}), y\right).
\end{equation}
In other words, to perform coreset selection for $\ell_p$-PGD~\citep{madry2018towards} and Perceptual~\citep{laidlaw2021pat} Adversarial Training, one needs to add a pre-processing step to the gradient computation.
At this step, adversarial examples for the entire training set must be constructed.
Then, the coresets can be built as in vanilla neural networks.

\subsubsection{\textcolor{blue}{Case 2} (TRADES Adversarial Training)}
For TRADES~\citep{zhang2019trades}, the gradient computation is slightly different as the objective in \Cref{eq:trades_functional} consists of two terms.
In this case, the gradient can be written as:
\begin{align}\label{eq:nn_gradeint_trades}\nonumber
	\nabla_{\boldsymbol{\theta}}\boldsymbol{\Phi} \left(\boldsymbol{x}, y; f_{\boldsymbol{\theta}}\right) &= \nabla_{\boldsymbol{\theta}}\mathcal{L}_{\mathrm{CE}}\left(f_{\boldsymbol{\theta}}(\boldsymbol{x}), y\right) \cdots \\ \cdots &+ \nabla_{\boldsymbol{\theta}}\max_{\tilde{\boldsymbol{x}}} \mathcal{L}_{\mathrm{CE}}\left(f_{\boldsymbol{\theta}}(\tilde{\boldsymbol{x}}), f_{\boldsymbol{\theta}}(\boldsymbol{x})\right)/\lambda,
\end{align}
The first term is the normal gradient of the neural network.
For the second term, we apply Danskin's theorem to obtain:
\begin{align}\label{eq:nn_gradeint_trades_penufinal}\nonumber
	\nabla_{\boldsymbol{\theta}}\boldsymbol{\Phi} \left(\boldsymbol{x}, y; f_{\boldsymbol{\theta}}\right) &\approx \nabla_{\boldsymbol{\theta}}\mathcal{L}_{\mathrm{CE}}\left(f_{\boldsymbol{\theta}}(\boldsymbol{x}), y\right)\cdots \\\cdots&+ \nabla_{\boldsymbol{\theta}} \mathcal{L}_{\mathrm{CE}}\left(f_{\boldsymbol{\theta}}(\boldsymbol{x}_{\rm{adv}}), f_{\boldsymbol{\theta}}(\boldsymbol{x})\right)/\lambda,
\end{align}
where $\boldsymbol{x}_{\rm{adv}}$ is an approximate solution to:
\begin{equation}\label{eq:trades_max}
\max_{\tilde{\boldsymbol{x}}} \mathcal{L}_{\mathrm{CE}}\left(f_{\boldsymbol{\theta}}(\tilde{\boldsymbol{x}}), f_{\boldsymbol{\theta}}(\boldsymbol{x})\right)/\lambda ~ \text{s.t.} ~ \mathrm{d}({\tilde{\boldsymbol{x}}, \boldsymbol{x}})\leq \varepsilon.
\end{equation}
Then, we compute the second gradient term in \Cref{eq:nn_gradeint_trades_penufinal} using the multi-variable chain rule~(see~\Cref{ap:trades_gradient}).
We can write the final TRADES gradient as:
	\begin{align}\nonumber
		&\nabla_{\boldsymbol{\theta}}\boldsymbol{\Phi} \left(\boldsymbol{x}, y; f_{\boldsymbol{\theta}}\right)\\\nonumber
		&\qquad= \nabla_{\boldsymbol{\theta}}\mathcal{L}_{\mathrm{CE}}\left(f_{\boldsymbol{\theta}}(\boldsymbol{x}), y\right)\cdots\\\nonumber
		&\qquad\cdots+ \nabla_{\boldsymbol{\theta}} \mathcal{L}_{\mathrm{CE}}\left(f_{\boldsymbol{\theta}}(\boldsymbol{x}_{\rm{adv}}), \text{\texttt{freeze}}\left(f_{\boldsymbol{\theta}}(\boldsymbol{x})\right)\right)/\lambda\cdots \\ \label{eq:nn_gradeint_trades_final}
		&\qquad\cdots+ \nabla_{\boldsymbol{\theta}} \mathcal{L}_{\mathrm{CE}}\left(\text{\texttt{freeze}}\left(f_{\boldsymbol{\theta}}(\boldsymbol{x}_{\rm{adv}})\right), f_{\boldsymbol{\theta}}(\boldsymbol{x})\right)/\lambda.
	\end{align}
where $\text{\texttt{freeze}}(\cdot)$ stops the gradients from backpropagating through its argument function.

Having found the loss gradients $\nabla_{\boldsymbol{\theta}}\boldsymbol{\Phi} \left(f_{\boldsymbol{\theta}}; \boldsymbol{x}_{i}, y_{i}\right)$ for $\ell_p$-PGD, PAT~(\textcolor{blue}{Case 1}), and TRADES~(\textcolor{blue}{Case 2}), we can construct \Cref{eq:gradient_app_coreset} and use existing greedy solvers like \textsc{Craig}~\citep{mirzasoleiman2020craig} or \textsc{GradMatch}~\citep{killamsetty2021gradmatch} to find the coreset.
Conceptually, adversarial coreset selection amounts to adding a pre-processing step where we need to build perturbed versions of the training data using their respective objectives in \Cref{eq:at_max_objective,eq:trades_max}.
Afterward, greedy subset selection algorithms are used to construct the coresets based on the value of the gradients.
Finally, having selected the coreset data, one can run a \textit{weighted adversarial training} only on the data that remains in the coreset:
\begin{equation}\label{eq:nn_training_objective_coreset_repeat}
    \min_{\boldsymbol{\theta}} \sum_{j \in \textcolor{mygreen}{\mathcal{S}^{*}}} \textcolor{mygreen}{\gamma^{*}_{j}} \boldsymbol{\Phi} \left(f_{\boldsymbol{\theta}}; \boldsymbol{x}_{j}, y_{j}\right).
\end{equation}
As can be seen, we are not changing the essence of the training objective in this process.
We are just reducing the training size to enhance the computational efficiency of our proposed solution, and as such, we can use it along any adversarial training objective.

\subsection{Practical Considerations}\label{sub:sub:practical}

\begin{algorithm*}[t!]
	\vspace*{0.25em}
	\caption{Adversarial Training with Coreset Selection~\label{alg:adv_core}} 
	\begin{small}
	\textbf{Input}: dataset $\mathcal{D}=\left\{\left(\boldsymbol{x}_{i}, y_{i}\right)\right\}_{i=1}^{n}$, neural network~${f_{\boldsymbol{\theta}}(\cdot)}$.\vspace*{0.25em}
	\\
	\textbf{Output}: robustly trained neural network~${f_{\boldsymbol{\theta}}(\cdot)}$.\vspace*{0.25em}
	\\
	\textbf{Parameters}: learning rate~$\alpha$, total epochs~$E$, warm-start coefficient~$\kappa$, coreset update period~$T$, batch size~$b$, coreset size~$k$, perturbation bound~$\varepsilon$.\vspace*{0.25em}
	\begin{algorithmic}[1]
		\State Initialize~${\boldsymbol{\theta}}$ randomly.
		\State $\kappa_{\rm epochs} = \kappa \cdot E$
		\State $T_{\mathrm{warm}} = \kappa_{\rm epochs} \cdot k$
		\For {$t=1,2,\ldots, E$}
		\If {$t \leq T_{\mathrm{warm}}$}\vspace*{0.15em}
		    \State $\mathcal{S}, \boldsymbol{\gamma} \leftarrow \mathcal{D}, \boldsymbol{1}$ \verb+\\ Warm-start with the entire data and uniform weights.+
		\ElsIf {$t \geq \kappa_{\rm epochs}$ \& $t\%T=0$} \vspace*{0.15em}
		    \State $\mathcal{I} = \text{\textsc{BatchAssignments}}\left(\mathcal{D}, b\right)$ \verb+\\ Batch-wise selection.+\vspace*{0.15em}
		    \State $\mathcal{Y} = \left\{{f_{\boldsymbol{\theta}}(\boldsymbol{x}_{i})}\mid\left(\boldsymbol{x}_{i}, y_{i}\right) \in \mathcal{D}\right\}$ \verb+\\ Computing the logits.+\vspace*{0.15em}
		    \State $\mathcal{G} = \text{\textsc{AdvGradient}}\left(\mathcal{D}, \mathcal{Y}\right)$ \verb+\\ Using Eqs. 18 & 22 to find the gradients.+\vspace*{0.15em}
		    \State $\mathcal{S}, \boldsymbol{\gamma}\leftarrow\text{\textsc{GreedySolver}}\left(\mathcal{D}, \mathcal{I}, \mathcal{G}, \mathrm{coreset~size}=k\right)$ \verb+\\ Using greedy solvers to get the coreset.+\vspace*{0.15em}
		\Else
		    \State \textbf{Continue}
		\EndIf
		\For {$\mathrm{batch}$ in $\mathcal{S}$}
		    \State $\mathrm{batch_{adv}}=\text{\textsc{AdvExampleGen}}\left(\mathrm{batch}, f_{\boldsymbol{\theta}}, \varepsilon\right)$ \vspace*{0.15em}
		    \State $\boldsymbol{\theta}\leftarrow\text{\textsc{SGD}}\left(\mathrm{batch_{adv}}, f_{\boldsymbol{\theta}}, \alpha, \boldsymbol{\gamma}\right)$ \verb+\\ Performing SGD over a batch of data.+\vspace*{0.15em}
		\EndFor
		\EndFor
	\end{algorithmic}
	\end{small}
\end{algorithm*}

Since coreset selection depends on the current values of the neural network weights, it is crucial to update the coresets as the training progresses.
Prior work~\citep{killamsetty2021glister,killamsetty2021gradmatch} has shown that this selection needs to be done every $T$ epochs, where $T$ is usually greater than 15.
Also, we employ small yet critical practical changes while using coreset selection to increase efficiency.
We summarize these practical tweaks below.
Further detail can be found in~\citep{killamsetty2021gradmatch,mirzasoleiman2020craig}.
\bmhead{Gradient Approximation}
    As we saw, both \Cref{eq:nn_gradient_at_final,eq:nn_gradeint_trades_final} require computation of the loss gradient with respect to the neural network weights.
	This is equal to backpropagation through the entire neural network, which is inefficient.
	Instead, it is common to replace the exact gradients in \Cref{eq:nn_gradient_at_final,eq:nn_gradeint_trades_final} with their last-layer approximation~\citep{katharopoulos2018notall,mirzasoleiman2020craig,killamsetty2021gradmatch}.
	In other words, instead of backpropagating through the entire network, one can backpropagate up until the penultimate layer.
	This estimate has an approximate complexity equal to forwardpropagation, and it has been shown to work well in practice~\citep{mirzasoleiman2020craig,mirzasoleiman2020crust,killamsetty2021glister,killamsetty2021gradmatch}.\footnote{\textcolor{revision}{Note that although adversarial coreset selection requires backpropagation for adversarial example generation, the gradient approximation technique comes in handy when we are trying to compute the point-wise distance between gradients of two samples. In particular, instead of having to compute the distance between two vectors with millions of elements, one can directly work on lower dimensional gradient estimates and calculate the distance between all samples in a dataset at once.}}
\bmhead{Batch-wise Coreset Selection}
	As discussed in \Cref{sec:sec:coreset_adversarial}, data selection is usually done in a \textit{sample-wise} fashion where each data sample is separately considered to be selected.
	This way, one must find the data candidates out of the entire training set.
	To increase efficiency, \citet{killamsetty2021gradmatch} proposed the \textit{batch-wise} variant.
	In this type of coreset selection, the data is first split into several batches.
	Then, the algorithm makes a selection out of these batches.
	Intuitively, this change can increase efficiency as the sample size is reduced from the number of data points to the number of batches.
\bmhead{Warm-start with the Entire Data}
	Finally, as we shall see in the experiments, it is important to warm-start the adversarial training using the entire dataset.
	Afterward, the coreset selection is activated, and adversarial training is only performed using the coreset data.

\subsection{Final Algorithm}

\Cref{fig:coresets_idea} and \Cref{alg:adv_core} summarize our coreset selection approach for adversarial training.
As can be seen, our proposed method is a generic and principled approach in contrast to existing methods such as FAT~\citep{wong2020fast}.
In particular, our approach provides the following advantages compared to existing methods:
\begin{enumerate}
	\item The proposed approach does not involve algorithmic level manipulations and dependency on specific training attributes such as $\ell_\infty$ bound or cyclic learning rate.
	Also, it controls the training speed through coreset size, which can be specified solely based on available computational resources.
	\item The simplicity of our method makes it compatible with any existing/future adversarial training objectives. Furthermore, as we will see in \Cref{sec:experiments}, our approach can be combined with any greedy coreset selection algorithms to deliver robust neural networks.
\end{enumerate}
These characteristics are important as they increase the likelihood of applying our proposed method for robust neural network training no matter the training objectives.
This starkly contrasts with existing methods focusing solely on a particular training objective.

\section{Experimental Results}\label{sec:experiments}
In this section, we present our experimental results.
We show how our proposed approach can efficiently reduce the training time of various robust objectives in different settings.
To this end, we train our approach using TRADES~\citep{zhang2019trades}, $\ell_p$-PGD~\citep{madry2018towards} and PAT~\citep{laidlaw2021pat} on CIFAR-10~\citep{krizhevsky2009learning}, SVHN~\citep{netzer2011reading}, and a subset of \textcolor{revision}{ImageNet~\citep{russakovsky2015imagenet} with 12 and 100 classes.
For TRADES and $\ell_p$-PGD training, we use ResNet-18~\citep{he2016deep} and WideResNet-28-10~\citep{zagoruyko2016wresnet} classifiers.
For PAT and ImageNet experiments, we use ResNet-34 and 50 architectures.}
Further implementation details can be found in~Appendix~\ref{ap:sec:imp_det}.

\subsection{TRADES and $\ell_p$-PGD Robust Training}

\begin{table*}[p!]
    \centering
	\caption{\textcolor{revision}{Clean~(ACC) and robust~(RACC) accuracy, and total training time~(T) of different adversarial training methods.
	For each method, all the hyper-parameters were kept the same as full training.
	For our proposed approach, the difference with full training is shown in parentheses.
	The information on the computation of RACC in each case is given in~Appendix~\ref{ap:sec:imp_det}.}}
	\label{tab:TRADES_lp_PGD}
	\begin{center}
		\begin{small}
		    \setlength\tabcolsep{0.45em}
			\def\arraystretch{1.65}
			\begin{tabular}{ccccccc}
				\toprule
                \multirow{2}{*}{\rotatebox[origin=c]{90}{\textbf{Model}}}
				&\multirow{2}{*}{\rotatebox[origin=c]{90}{\textbf{Objecive}}}
				&\multirow{2}{*}{\rotatebox[origin=c]{90}{\textbf{Data}}}
				&\multirow{2}{*}{\textbf{Training}}
				&\multicolumn{3}{c}{\textbf{Performance Measures}}\\
				\cmidrule(lr){5-7}
				&&&                                 &  \textuparrow~\textbf{ACC} (\%)         & \textuparrow~\textbf{RACC} (\%)         & \textdownarrow~\textbf{T} (mins)\\
				\midrule
                \multirow{9}{*}{\rotatebox[origin=c]{90}{\textbf{ResNet-18}}}
				& \multirow{3}{*}{\rotatebox[origin=c]{90}{\textbf{TRADES}}}  & \multirow{3}{*}{\rotatebox[origin=c]{90}{\textbf{CIFAR-10}}}
				&  Adv. \textsc{Craig} (Ours)       & $83.03$ (\textcolor{red}{$-2.38$})            & $41.45$ (\textcolor{red}{$-2.74$})            &  $179.20$ (\textcolor{darkgreen}{$-165.09$})\\
				&&& Adv. \textsc{GradMatch} (Ours)  & $83.07$ (\textcolor{red}{$-2.34$})            & $41.52$ (\textcolor{red}{$-2.67$})            &  $178.73$ (\textcolor{darkgreen}{$-165.56$})\\
				&&& Full Adv. Training              & $85.41$                                       & $44.19$                                       &  $344.29$\\
				\cmidrule(lr){2-7}
				& \multirow{3}{*}{\rotatebox[origin=c]{90}{\textbf{$\ell_\infty$-PGD}}} & \multirow{3}{*}{\rotatebox[origin=c]{90}{\textbf{CIFAR-10}}}
				&  Adv. \textsc{Craig} (Ours)       & $80.37$ (\textcolor{red}{$-2.77$})            & $45.07$ (\textcolor{darkgreen}{$+3.68$})            &  $148.01$ (\textcolor{darkgreen}{$-144.86$})\\
				&&& Adv. \textsc{GradMatch} (Ours)  & $80.67$ (\textcolor{red}{$-2.47$})            & $45.23$ (\textcolor{darkgreen}{$+3.84$})            &  $148.03$ (\textcolor{darkgreen}{$-144.84$})\\
				&&& Full Adv. Training              & $83.14$                                       & $41.39$                                             &  $292.87$\\
				\cmidrule(lr){2-7}
				& \multirow{3}{*}{\rotatebox[origin=c]{90}{\textbf{$\ell_2$-PGD}}} & \multirow{3}{*}{\rotatebox[origin=c]{90}{\textbf{SVHN}}}
				&  Adv. \textsc{Craig} (Ours)       & $95.42$ (\textcolor{darkgreen}{$+0.10$})            & $49.68$ (\textcolor{red}{$-3.34$})            &  $130.04$ (\textcolor{darkgreen}{$-259.42$})\\
				&&& Adv. \textsc{GradMatch} (Ours)  & $95.57$ (\textcolor{darkgreen}{$+0.25$})            & $50.41$ (\textcolor{red}{$-2.61$})            &  $125.53$ (\textcolor{darkgreen}{$-263.93$})\\
				&&& Full Adv. Training              & $95.32$                                             & $53.02$                                       &  $389.46$\\
                \midrule
                \multirow{9}{*}{\rotatebox[origin=c]{90}{\textbf{WideResNet-28-10}}}
				& \multirow{3}{*}{\rotatebox[origin=c]{90}{\textbf{TRADES}}}  & \multirow{3}{*}{\rotatebox[origin=c]{90}{\textbf{CIFAR-10}}}
				&  Adv. \textsc{Craig} (Ours)       & $85.25$ (\textcolor{red}{$-1.50$})            & $43.79$ (\textcolor{red}{$-1.29$})            &  $1022.52$ (\textcolor{darkgreen}{$-1037.10$})\\
				&&& Adv. \textsc{GradMatch} (Ours)  & $84.54$ (\textcolor{red}{$-2.21$})            & $47.83$ (\textcolor{darkgreen}{$+2.75$})      &  $1032.32$ (\textcolor{darkgreen}{$-1027.30$})\\
				&&& Full Adv. Training              & $86.75$                                       & $45.08$                                       &  $2059.62$\\
				\cmidrule(lr){2-7}
				& \multirow{3}{*}{\rotatebox[origin=c]{90}{\textbf{$\ell_\infty$-PGD}}} & \multirow{3}{*}{\rotatebox[origin=c]{90}{\textbf{CIFAR-10}}}
				&  Adv. \textsc{Craig} (Ours)       & $83.57$ (\textcolor{red}{$-2.68$})            & $41.78$ (\textcolor{red}{$-3.71$})                  &  $1364.99$ (\textcolor{darkgreen}{$-1348.18$})\\
				&&& Adv. \textsc{GradMatch} (Ours)  & $83.47$ (\textcolor{red}{$-2.78$})            & $42.16$ (\textcolor{red}{$-3.33$})                  &  $1364.45$ (\textcolor{darkgreen}{$-1348.72$})\\
				&&& Full Adv. Training              & $86.25$                                       & $45.49$                                             &  $2713.17$\\
				\cmidrule(lr){2-7}
				& \multirow{3}{*}{\rotatebox[origin=c]{90}{\textbf{$\ell_2$-PGD}}} & \multirow{3}{*}{\rotatebox[origin=c]{90}{\textbf{SVHN}}}
				&  Adv. \textsc{Craig} (Ours)       & $95.96$ (\textcolor{darkgreen}{$+0.19$})            & $43.04$ (\textcolor{red}{$-2.21$})            &  $886.64$ (\textcolor{darkgreen}{$-1826.99$})\\
				&&& Adv. \textsc{GradMatch} (Ours)  & $96.06$ (\textcolor{darkgreen}{$+0.29$})            & $45.47$ (\textcolor{darkgreen}{$+0.22$})      &  $874.14$ (\textcolor{darkgreen}{$-1839.49$})\\
				&&& Full Adv. Training              & $95.77$                                             & $45.25$                                       &  $2713.63$\\
                \midrule
                \multirow{3}{*}{\rotatebox[origin=c]{90}{\textbf{ResNet-50}}}
                &\multirow{3}{*}{\rotatebox[origin=c]{90}{\textbf{TRADES}}}  & \multirow{3}{*}{\rotatebox[origin=c]{90}{\footnotesize\textbf{ImageNet-100}}}
				&  Adv. \textsc{Craig} (Ours)      & $56.64$ (\textcolor{red}{$-2.98$})            & $21.80$ (\textcolor{red}{$-1.94$})            &  $1369.21$ (\textcolor{darkgreen}{$-997.28$})\\
				&&& Adv. \textsc{GradMatch} (Ours)  & $56.48$ (\textcolor{red}{$-3.44$})            & $21.92$ (\textcolor{red}{$-1.82$})            &  $1360.14$ (\textcolor{darkgreen}{$-1006.35$})\\
				&&& Full Adv. Training              & $59.92$                                       & $23.74$                                       &  $2366.49$\\
                \midrule
                \multirow{3}{*}{\rotatebox[origin=c]{90}{\textbf{ResNet-34}}}
                &\multirow{3}{*}{\rotatebox[origin=c]{90}{\textbf{$\ell_\infty$-PGD}}}  & \multirow{3}{*}{\rotatebox[origin=c]{90}{\footnotesize\textbf{ImageNet-100}}}
				&  Adv. \textsc{Craig} (Ours)       & $65.52$ (\textcolor{red}{$-1.58$})            & $40.00$ (\textcolor{red}{$-0.52$})            &  $420.81$ (\textcolor{darkgreen}{$-434.11$})\\
				&&& Adv. \textsc{GradMatch} (Ours)  & $65.18$ (\textcolor{red}{$-1.92$})            & $40.72$ (\textcolor{darkgreen}{$+0.20$})      &  $403.66$ (\textcolor{darkgreen}{$-451.26$})\\
				&&& Full Adv. Training              & $67.10$                                       & $40.52$                                       &  $854.92$\\
				\bottomrule
			\end{tabular}
		\end{small}
	\end{center}
\end{table*}

\textcolor{revision}{In our first set of experiments, we train well-known neural network classifiers on CIFAR-10, SVHN, and ImageNet-100 datasets using TRADES, $\ell_\infty$ and $\ell_2$-PGD adversarial training objectives.}
In each case, we set the training hyper-parameters such as the learning rate, the number of epochs, and attack parameters.
Then, we train the network using the entire training data and our adversarial coreset selection approach.
For our approach, we use batch-wise versions of \textsc{Craig}~\citep{mirzasoleiman2020craig} and \textsc{GradMatch}~\citep{killamsetty2021gradmatch} with warm-start.
We set the \textit{coreset size} (the percentage of training data to be selected) to 50\% for CIFAR-10 and ImageNet-100, and 30\% for SVHN to get a reasonable balance between accuracy and training time.
We report the clean and robust accuracy (in \%) as well as the total training time (in minutes) in \Cref{tab:TRADES_lp_PGD}.
For our approach, we also report the difference with full training in parenthesis.
In each case, we evaluate the robust accuracy using an attack with similar attributes as the training objective.

As can be seen in \Cref{tab:TRADES_lp_PGD}, in most cases, we reduce the training time by more than a factor of two, while keeping the clean and robust accuracy almost intact.
Note that in these experiments, all the training attributes such as the hyper-parameters, learning rate scheduler, etc.~are the same among different training schemes.
This is important since we want to clearly show the relative boost in performance that one can achieve just by using coreset selection.
Nonetheless, it is likely that by tweaking the hyper-parameters for our approach, one can obtain even better results in terms of clean and robust accuracy.\footnote{\textcolor{revision}{As reported in \citet{qin2019adversarial}, we also observed that TRADES~\citep{zhang2019trades} does not provide optimal results for the ImageNet dataset. Nevertheless, our results are relative, and if a better set of hyper-parameters could be found for this setting, the performance of adversarial coreset selection would also be improved.}}

\subsection{Perceptual Adversarial Training vs. Unseen Attacks}
As discussed in \Cref{sec:background}, PAT~\citep{laidlaw2021pat} replaces the visual similarity measure~$\mathrm{d}(\cdot, \cdot)$ in \Cref{eq:at_functional} with LPIPS~\citep{zhang2018lpips} distance.
The logic behind this choice is that $\ell_p$ norms can only capture a small portion of images similar to the clean one, limiting the search space of adversarial attacks.
Motivated by this reason, \citet{laidlaw2021pat} proposes two different ways of finding the solution to \Cref{eq:at_functional} when $\mathrm{d}(\cdot, \cdot)$ is the LPIPS distance.
The first version uses PGD, and the second is a relaxation of the original problem using the Lagrangian form.
We refer to these two versions as PPGD~(Perceptual PGD) and LPA~(Lagrangian Perceptual Attack), respectively.
Then, \citet{laidlaw2021pat} proposed to utilize a fast version of LPA to enable its efficient usage in adversarial training.
More information on this approach can be found in \citep{laidlaw2021pat}.

\begin{table*}[t!]
	\caption{Clean (ACC) and robust (RACC) accuracy, and total training time (T) of Perceptual Adversarial Training for CIFAR-10 and ImageNet-12 datasets. At inference, the networks are evaluated against five attacks that were not seen during training (Unseen RACC), as well as different versions of Perceptual Adversarial Attack (Seen RACC). In each case, the average is reported. For more information and details about the experiment, please see~Appendices~\ref{ap:sec:imp_det} and \ref{ap:extended}.}
	\label{tab:LPA_sum}
	\begin{center}
		\begin{small}
			\setlength\tabcolsep{0.3em}
			\def\arraystretch{1.75}
			\begin{tabular}{ccccccc}
				\toprule
				\multirow{2}{*}{\rotatebox[origin=c]{90}{\textbf{Data}}}
				& \multirow{2}{*}{\textbf{Training}}
				& \multirow{2}{*}{\textuparrow~\textbf{ACC} (\%)}
				& \multicolumn{2}{c}{\textuparrow~\textbf{RACC} (\%)}
				& \multirow{2}{*}{\textdownarrow~\textbf{T} (mins)}\\
				\cmidrule{4-5}
				& & & Unseen & Seen\\
				\midrule
				\multirow{3}{*}{\rotatebox[origin=c]{90}{\textbf{CIFAR-10}}}
				& Adv. \textsc{Craig} (Ours)             & $83.21$ (\textcolor{red}{$-2.81$})         & $46.55$ (\textcolor{red}{$-1.49$})    & $13.49$ (\textcolor{red}{$-1.83$})    & $767.34$ (\textcolor{darkgreen}{$-915.60$}) \\
				& Adv. \textsc{GradMatch} (Ours)         & $83.14$ (\textcolor{red}{$-2.88$}) 	      & $46.11$ (\textcolor{red}{$-1.93$})    & $13.74$ (\textcolor{red}{$-1.54$})    & $787.26$ (\textcolor{darkgreen}{$-895.68$}) \\
				& Full PAT (Fast-LPA)                    & $86.02$	                                  & $48.04$                               & $15.32$                               & $1682.94$ \\
				\midrule
				\multirow{3}{*}{\rotatebox[origin=c]{90}{\footnotesize{\textbf{ImageNet-12}}}}
				& Adv. \textsc{Craig} (Ours)             & $86.99$ (\textcolor{red}{$-4.23$})         & $53.05$ (\textcolor{red}{$-0.18$})    & $22.56$ (\textcolor{red}{$-0.77$})   & $2817.06$ (\textcolor{darkgreen}{$-2796.06$}) \\
				& Adv. \textsc{GradMatch} (Ours)         & $87.08$ (\textcolor{red}{$-4.14$}) 	      & $53.17$ (\textcolor{red}{$-0.06$})    & $20.74$ (\textcolor{red}{$-2.59$})   & $2865.72$ (\textcolor{darkgreen}{$-2747.40$}) \\
				& Full PAT (Fast-LPA)                    & $91.22$	                                  & $53.23$                               & $23.33$                              & $5613.12$ \\
				\bottomrule
			\end{tabular}
		\end{small}
	\end{center}
\end{table*}

For our next set of experiments, we show how our approach can be adapted to this unusual training objective.
This is done to showcase the compatibility of our proposed method with different training objectives as opposed to existing methods that are carefully tuned for a particular training objective.
To this end, we train ResNet-50 classifiers using Fast-LPA.
In this case, we train the classifiers on CIFAR-10 and ImageNet-12 datasets.
Like our previous experiments, we set the hyper-parameters of the training to be fixed and then train the models using the entire training data and our adversarial coreset selection method.
For our method, we use batch-wise versions of \textsc{Craig}~\citep{mirzasoleiman2020craig} and \textsc{GradMatch}~\citep{killamsetty2021gradmatch} with warm-start.
The \textit{coreset size} for CIFAR-10 and ImageNet-12 were set to 40\% and 50\%, respectively.
As in~\citet{laidlaw2021pat}, we measure the performance of the trained models against unseen attacks during training, as well as the two variants of perceptual attacks.
The unseen attacks for each dataset were selected in a similar manner to~\citet{laidlaw2021pat}, and the attack parameters can be found in~Appendix~\ref{ap:sec:imp_det}.
We also record the total training time taken by each method.

\Cref{tab:LPA_sum} summarizes our results on PAT using Fast-LPA (full results can be found in~Appendix~\ref{ap:extended}).
As seen, our adversarial coreset selection approach can deliver a competitive performance in terms of clean and average unseen attack accuracy while reducing the training time by at least a factor of two.
These results indicate the flexibility of our adversarial coreset selection that can be combined with various objectives.
This is due to the orthogonality of the proposed approach with the existing efficient adversarial training methods.
In this case, we can make Fast-LPA even faster by using our approach.

\subsection{Compatibility with Existing Methods}\label{sec:sec:comaptibility}
\textcolor{revision}{To showcase that our adversarial coreset selection approach is complementary to existing methods, we integrate it with two existing baselines that aim to improve the efficiency of adversarial training.}
\bmhead{Early Termination}
\textcolor{revision}{Going through our results in~\Cref{tab:TRADES_lp_PGD,tab:LPA_sum}, one might wonder what would happen if we decrease the number of training epochs by half.
To perform this experiment, we select the WideResNet-28-10 architecture and train robust neural networks over CIFAR-10 and SVHN datasets.
We set all our hyper-parameters in a similar manner to the ones used for the experiments in~\Cref{tab:TRADES_lp_PGD}, and only halve the number of training epochs.
To make sure that the learning rate is also comparable, we halve the learning rate scheduler epochs as well.
Then, we train the neural networks using $\ell_\infty$ and $\ell_2$-PGD adversarial training.}

\textcolor{revision}{\Cref{tab:early_stopping} shows our results compared to the ones reported in~\Cref{tab:TRADES_lp_PGD}.
As can be seen, adversarial coreset selection obtains a similar performance to using the entire data by consuming 2-3 times less training time.}

\bmhead{Fast Adversarial Training}
\textcolor{revision}{Additionally, we integrate adversarial coreset selection with a stable version of Fast Adversarial Training~(FAT)~\citep{wong2020fast} that does not use cyclic learning rate.}
Specifically, we train a neural network using FAT~\citep{wong2020fast}, and then add adversarial coreset selection to this approach and record the training time and clean and robust accuracy.
We run the experiments on the CIFAR-10 dataset and train a ResNet-18 for each case.
We set our methods' coreset size to 50\%.
The results are shown in \Cref{tab:FAT}.
As seen, our approach can be easily combined with existing methods to deliver faster training.
This is due to the orthogonality of our approach with existing methods that we discussed previously.

\begin{table*}[t!]
    \centering
	\caption{\textcolor{revision}{Clean~(ACC) and robust~(RACC) accuracy, and total training time~(T) of different adversarial training methods over WideResNet-28-10.
	For each method, all the hyper-parameters were kept the same as \Cref{tab:TRADES_lp_PGD}.
    The only exception is that all the epoch-related parameters were halved.
	The difference with full training is shown in parentheses for our proposed approach.
	The information on the computation of RACC in each case is given in~Appendix~\ref{ap:sec:imp_det}.}}
	\label{tab:early_stopping}
	\begin{center}
		\begin{small}
		    \setlength\tabcolsep{0.45em}
			\def\arraystretch{1.25}
			\begin{tabular}{cccccc}
				\toprule
                \multirow{2}{*}{\rotatebox[origin=c]{90}{\footnotesize\textbf{Objective}}}
				&\multirow{2}{*}{\rotatebox[origin=c]{90}{\textbf{Data}}}
				&\multirow{2}{*}{\textbf{Training}}
				&\multicolumn{3}{c}{\textbf{Performance Measures}}\\
				\cmidrule(lr){4-6}
				&&                                 &  \textuparrow~\textbf{ACC} (\%)         & \textuparrow~\textbf{RACC} (\%)         & \textdownarrow~\textbf{T} (mins)\\
				\midrule
				\multirow{3}{*}{\rotatebox[origin=c]{90}{\textbf{$\ell_\infty$-PGD}}} & \multirow{3}{*}{\rotatebox[origin=c]{90}{\footnotesize\textbf{CIFAR-10}}}
				&  Adv. \textsc{Craig} (Ours)      & $83.18$ (\textcolor{red}{$-3.61$})            & $48.90$ (\textcolor{darkgreen}{$+0.04$})            &  $517.39$ (\textcolor{darkgreen}{$-506.70$})\\
				&& Adv. \textsc{GradMatch} (Ours)  & $84.28$ (\textcolor{red}{$-2.51$})            & $48.98$ (\textcolor{darkgreen}{$+0.12$})            &  $523.18$ (\textcolor{darkgreen}{$-500.91$})\\
				&& Full Adv. Training              & $86.79$                                       & $48.86$                                             &  $1024.09$\\
				\midrule
				\multirow{3}{*}{\rotatebox[origin=c]{90}{\textbf{$\ell_2$-PGD}}} & \multirow{3}{*}{\rotatebox[origin=c]{90}{\footnotesize\textbf{SVHN}}}
				&  Adv. \textsc{Craig} (Ours)      & $96.11$ (\textcolor{red}{$-0.02$})                  & $48.91$ (\textcolor{darkgreen}{$+1.22$})             &  $361.24$ (\textcolor{darkgreen}{$-1070.05$})\\
				&& Adv. \textsc{GradMatch} (Ours)  & $96.25$ (\textcolor{darkgreen}{$+0.12$})            & $47.88$ (\textcolor{darkgreen}{$+0.19$})             &  $394.40$ (\textcolor{darkgreen}{$-1036.89$})\\
				&& Full Adv. Training              & $96.13$                                             & $47.69$                                              &  $1431.29$\\
                \bottomrule
			\end{tabular}
		\end{small}
	\end{center}
    \vspace{-0.25in}
\end{table*}

\begin{table*}[t!]
	\caption{Clean~(ACC) and robust~(RACC) accuracy, and average training speed~(\textbf{S}\textsubscript{avg}) of Fast Adversarial Training~\citep{wong2020fast} without and with our adversarial coreset selection on CIFAR-10.
	For our proposed approach, the difference with full training is shown in parentheses.}
	\label{tab:FAT}
	\begin{center}
		\begin{small}
		    \setlength\tabcolsep{0.35em}
			\def\arraystretch{1.25}
			\begin{tabular}{lcccc}
				\toprule
				\multirow{2}{*}{\textbf{Training}}
				&\multicolumn{4}{c}{\textbf{Performance Measures}}\\
				\cmidrule(lr){2-5}
				&  \textuparrow~\textbf{ACC} (\%)                           & \textuparrow~\textbf{RACC} (\%)                           & \textdownarrow~\textbf{S}\textsubscript{avg} (min/epoch)   & \textdownarrow~\textbf{T} (mins)\\
				\midrule
				Fast Adv. Training                 & $86.20$                                       & $47.54$                                       &  $0.5178$  &  $31.068$\\
				~+ Adv. \textsc{Craig} (Ours)      & $82.56$ (\textcolor{red}{$-3.64$})            & $47.77$ (\textcolor{darkgreen}{$+0.23$})      &  $0.2783$  &  $16.695$ (\textcolor{darkgreen}{$-14.373$})\\
				~+ Adv. \textsc{GradMatch} (Ours)   & $82.53$ (\textcolor{red}{$-3.67$})           & $47.88$ (\textcolor{darkgreen}{$+0.34$})      &  $0.2737$  &  $16.419$ (\textcolor{darkgreen}{$-14.649$})\\
				\bottomrule
			\end{tabular}
		\end{small}
	\end{center}
\end{table*}
    
\begin{figure}
	    \centering
        \includegraphics[width=0.90\linewidth]{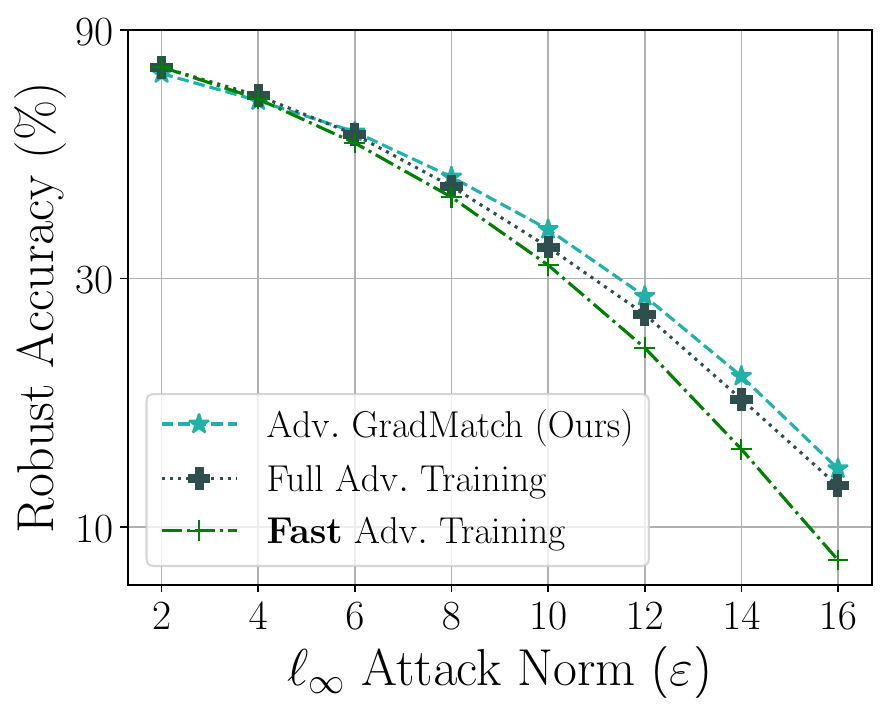}
        \caption{\textcolor{revision}{Robust accuracy as a function of $\ell_\infty$ attack norm. We train neural networks with a perturbation norm of $\norm{\varepsilon}_\infty \leq 8$ on CIFAR-10. At inference, we evaluate the robust accuracy against PGD-50 with various attack strengths.}}
		\label{fig:fat}
\end{figure}

Moreover, we show that adversarial coreset selection gives a better approximation to $\ell_\infty$-PGD adversarial training compared to using FGSM~\citep{goodfellow2014explaining} as done in FAT~\citep{wong2020fast}.
To this end, we use our adversarial \textsc{GradMatch} to train neural networks with the original $\ell_\infty$-PGD objective.
We also train these networks using FAT~\citep{wong2020fast} that uses FGSM.
We train neural networks with a perturbation norm of $\norm{\varepsilon}_\infty \leq 8$.
Then, we evaluate the trained networks against PGD-50 adversarial attacks with different attack strengths to see how each network generalizes to unseen perturbations.
As seen in \Cref{fig:fat}, adversarial coreset selection is a closer approximation to $\ell_\infty$-PGD compared to FAT~\citep{wong2020fast}.
This indicates the success of the proposed approach in retaining the characteristics of the original objective as opposed to existing methods like~\citep{wong2020fast,andriushchenko2020understanding}.

\begin{table*}[tb!]
		\caption{Performance of $\ell_\infty$-PGD. In ``Half-Half'', we mix half adversarial coreset selection samples with another half of clean samples and train a neural network similar to~\citep{tsipras2019robustness}.
		In ``ONLY-Core'' we just use adversarial coreset samples. Settings are given in \Cref{tab:hyperI}.
		The results are averaged over 5 runs.}
		\label{tab:HH_PGD}
		\begin{center}
			\begin{small}
				\setlength\tabcolsep{0.45em}
				\def\arraystretch{1.1}
				\begin{tabular}{cccccccc}
					\toprule
					\multirow{2}{*}{\textbf{Training}}  &  \multicolumn{2}{c}{\textuparrow~\textbf{Clean} (\%)}     & \multicolumn{2}{c}{\textuparrow~\textbf{RACC} (\%)} & \multicolumn{2}{c}{\textuparrow~\textbf{T} (mins)}\\
					\cmidrule(lr){2-3}\cmidrule(lr){4-5}\cmidrule(lr){6-7}
					                                    & \textbf{ONLY Core} & \textbf{Half-Half}                   & \textbf{ONLY Core} & \textbf{Half-Half}             & \textbf{ONLY Core} & \textbf{Half-Half}\\
					\midrule
					Adv. \textsc{Craig}                 & $80.36$    & $84.43$                                      & $45.07$  & $39.83$         & $148.01$ & $152.34$\\
					Adv. \textsc{GradMatch}             & $80.67$    & $84.31$                                      & $45.23$  & $40.05$         & $148.03$ & $153.18$\\
					Full Adv. Training                  & \multicolumn{2}{c}{$83.14$}                               & \multicolumn{2}{c}{$41.39$}  & \multicolumn{2}{c}{$292.87$}     \\
					\bottomrule
				\end{tabular}
			\end{small}
		\end{center}
\end{table*}
\begin{figure*}
	    \centering
        \includegraphics[width=0.70\linewidth]{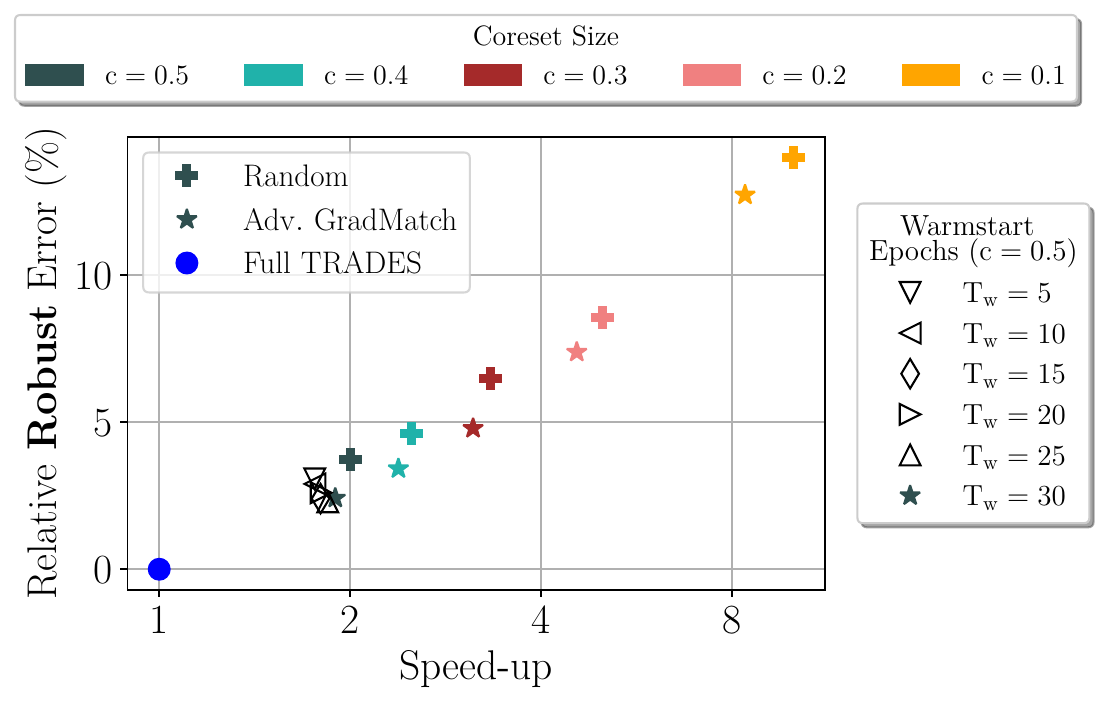}
        \caption{Relative robust error vs. speed up for TRADES. For a given subset size, we compare our adversarial coreset selection (\textsc{GradMatch}) against random data selection. Furthermore, we show our results for a selection of different warm-start settings.}
		\label{fig:rand}
\end{figure*}

\subsection{Ablation Studies}
In this section, we perform a few ablation studies to examine the effectiveness of our adversarial coreset selection method.
In our first set of experiments, we compare a random data selection with adversarial \textsc{GradMatch}.
\Cref{fig:rand} shows that for any given coreset size, our adversarial coreset selection method results in a lower robust error.
Furthermore, we modify the warm-start epochs for a fixed coreset size of 50\%.
The proposed method is not very sensitive to the number of warm-start epochs, although a longer warm-start is generally beneficial.

\begin{table*}[t!]
    \centering
	\caption{\textcolor{revision}{Clean~(ACC) and robust~(RACC) accuracy, and total training time~(T) of $\ell_\infty$-PGD adversarial training over CIFAR-10 for WideResNet-28-10 architecture.
	For each method, all the hyper-parameters were kept the same as \Cref{tab:early_stopping}.
    The frequency column indicates the number of epoch that we wait to update the coreset using~\Cref{alg:adv_core}.
	The information on the computation of RACC in each case is given in~Appendix~\ref{ap:sec:imp_det}.}}
	\label{tab:freq}
	\begin{center}
		\begin{small}
		    \setlength\tabcolsep{0.45em}
			\def\arraystretch{1.25}
			\begin{tabular}{ccccc}
				\toprule
                \multirow{2}{*}{\rotatebox[origin=c]{90}{\shortstack{\textbf{Freq.}\\(epoch)}}}
                &\multirow{2}{*}{\textbf{Training}}
				&\multicolumn{3}{c}{\textbf{Performance Measures}}\\
				\cmidrule(lr){3-5}
				&                                &  \textuparrow~\textbf{ACC} (\%)         & \textuparrow~\textbf{RACC} (\%)         & \textdownarrow~\textbf{T} (mins)\\
				\midrule
                \multirow{2}{*}{\rotatebox[origin=c]{90}{1}}
				&Adv. \textsc{Craig} (Ours)      & $84.55$ (\textcolor{red}{$-2.24$})          & $50.06$ (\textcolor{darkgreen}{$+1.20$})         &  $576.42$ (\textcolor{darkgreen}{$1.77\times$}) \\
				&Adv. \textsc{GradMatch} (Ours)  & $83.20$ (\textcolor{red}{$-3.59$})          & $49.77$ (\textcolor{darkgreen}{$+0.91$})         &  $582.37$ (\textcolor{darkgreen}{$1.76\times$}) \\
				\midrule
                \multirow{2}{*}{\rotatebox[origin=c]{90}{5}}
				&Adv. \textsc{Craig} (Ours)      & $83.77$ (\textcolor{red}{$-3.02$})          & $50.25$ (\textcolor{darkgreen}{$+1.39$})         &  $522.98$ (\textcolor{darkgreen}{$1.96\times$}) \\
				&Adv. \textsc{GradMatch} (Ours)  & $83.92$ (\textcolor{red}{$-2.87$})          & $49.12$ (\textcolor{darkgreen}{$+0.26$})         &  $528.24$ (\textcolor{darkgreen}{$1.94\times$}) \\
                \midrule
                \multirow{2}{*}{\rotatebox[origin=c]{90}{10}}
    			&Adv. \textsc{Craig} (Ours)      & $83.18$ (\textcolor{red}{$-3.61$})          & $48.90$ (\textcolor{darkgreen}{$+0.04$})         &  $517.40$ (\textcolor{darkgreen}{$1.98\times$}) \\
				&Adv. \textsc{GradMatch} (Ours)  & $84.28$ (\textcolor{red}{$-2.51$})          & $48.98$ (\textcolor{darkgreen}{$+0.12$})         &  $523.18$ (\textcolor{darkgreen}{$1.96\times$}) \\
                \midrule
				- &Full Adv. Training              & $86.79$                                     & $48.86$                                          &  $1024.09$\\
                \bottomrule
			\end{tabular}
		\end{small}
	\end{center}
\end{table*}

In another comparison, we run an experiment similar to that of \citet{tsipras2019robustness}.
Specifically, we minimize the average of adversarial and vanilla training in each epoch.
The non-coreset data is treated as clean samples to minimize the vanilla objective, while for the coreset samples, we would perform adversarial training.
\Cref{tab:HH_PGD} shows the results of this experiment.
As seen, adding the non-coreset data as clean inputs to the training improves the clean accuracy while decreasing the robust accuracy.
These results align with the observations of \citet{tsipras2019robustness} around the existence of a trade-off between clean and robust accuracy.

\textcolor{revision}{Next, we investigate the effect of adversarial coreset selection frequency.
Remember from~\Cref{sub:sub:practical} where we argued that performing adversarial coreset selection every $T$ epochs would help with the speed-up.
However, one must note that setting $T$ to a large number might come at the cost of sacrificing clean and robust accuracy.
To show this, we perform our early stopping experiments from~\Cref{tab:early_stopping} with different coreset renewal frequencies.
Our results are given in~\Cref{tab:freq}.
As can be seen, decreasing coreset selection renewal frequency would be helpful in gaining more speed-up but it could hurt the overall model performance.}

Finally, we study the accuracy vs.~speed-up trade-off in different versions of adversarial coreset selection.
For this study, we train our adversarial coreset selection method using different versions of \textsc{Craig}~\citep{mirzasoleiman2020craig} and \textsc{GradMatch}~\citep{killamsetty2021gradmatch} on CIFAR-10 using TRADES.
In particular, for each method, we start with the base algorithm and add the batch-wise selection and warm-start one by one.
Also, to capture the effect of the coreset size, we vary this number from 50\% to 10\%.
\Cref{fig:tradeoff} shows the clean and robust error vs.~speed-up compared to full adversarial training.
In each case, the combination of warm-start and batch-wise versions of the adversarial coreset selection gives the best performance.
Moreover, as we gradually decrease the coreset size, the training speed goes up.
However, this gain in training speed is achieved at the cost of increasing the clean and robust error.
Both of these observations align with that of~\citet{killamsetty2021gradmatch} around vanilla coreset selection.

\begin{figure*}[t!]
	\centering
	\begin{subfigure}{.45\textwidth}
		\centering
		\includegraphics[width=1.0\textwidth]{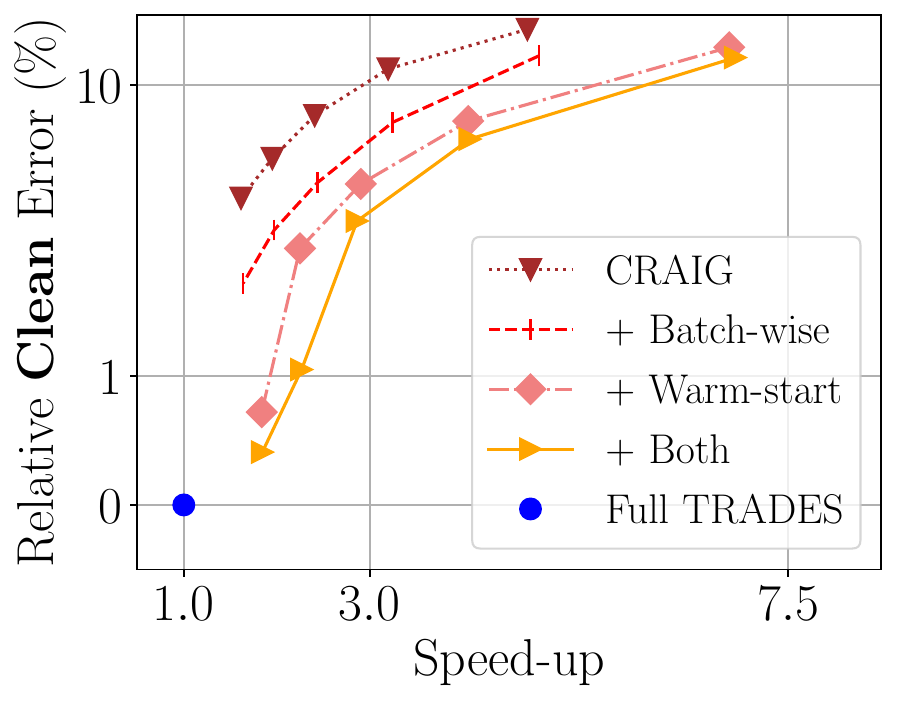}
		\caption{}
		\label{fig:tradeoff:craig_clean}
	\end{subfigure}\hspace*{2em}
	\begin{subfigure}{.45\textwidth}
		\centering
		\includegraphics[width=1.0\textwidth]{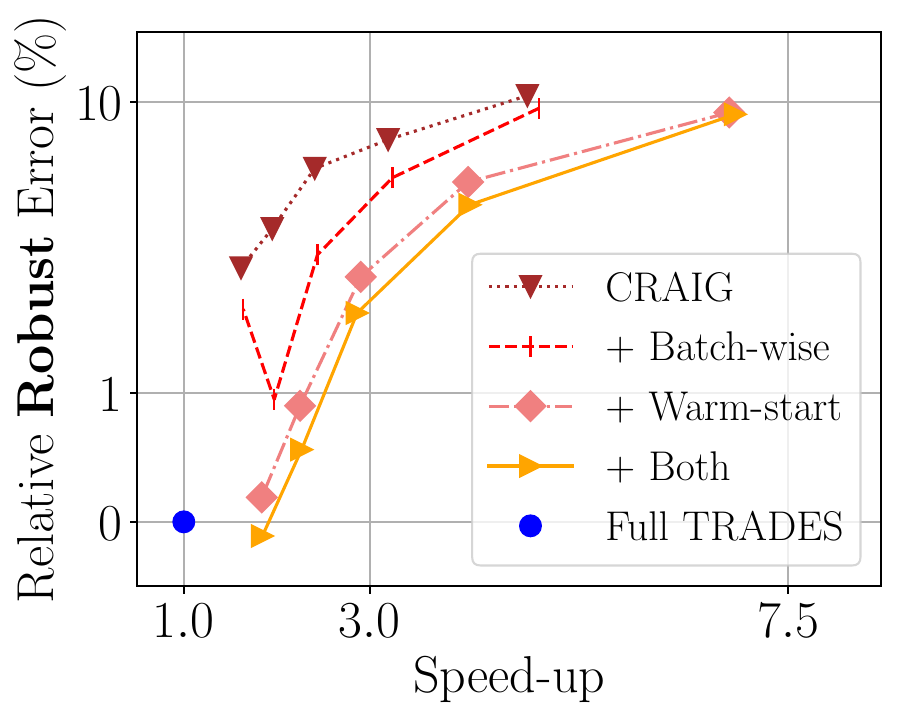}
		\caption{}
		\label{fig:tradeoff:craig_robust}
	\end{subfigure}\\\vspace*{0.2em}
	\begin{subfigure}{.45\textwidth}
		\centering
		\includegraphics[width=1.0\textwidth]{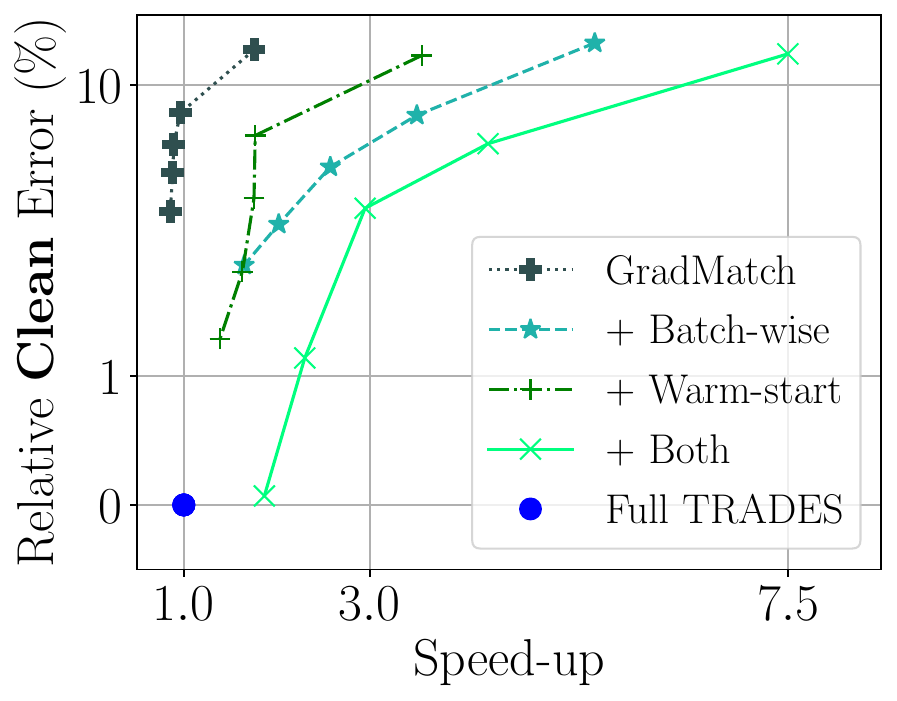}
		\caption{}
		\label{fig:tradeoff:gradmatch_clean}
	\end{subfigure}\hspace*{2em}
	\begin{subfigure}{.45\textwidth}
		\centering
		\includegraphics[width=1.0\textwidth]{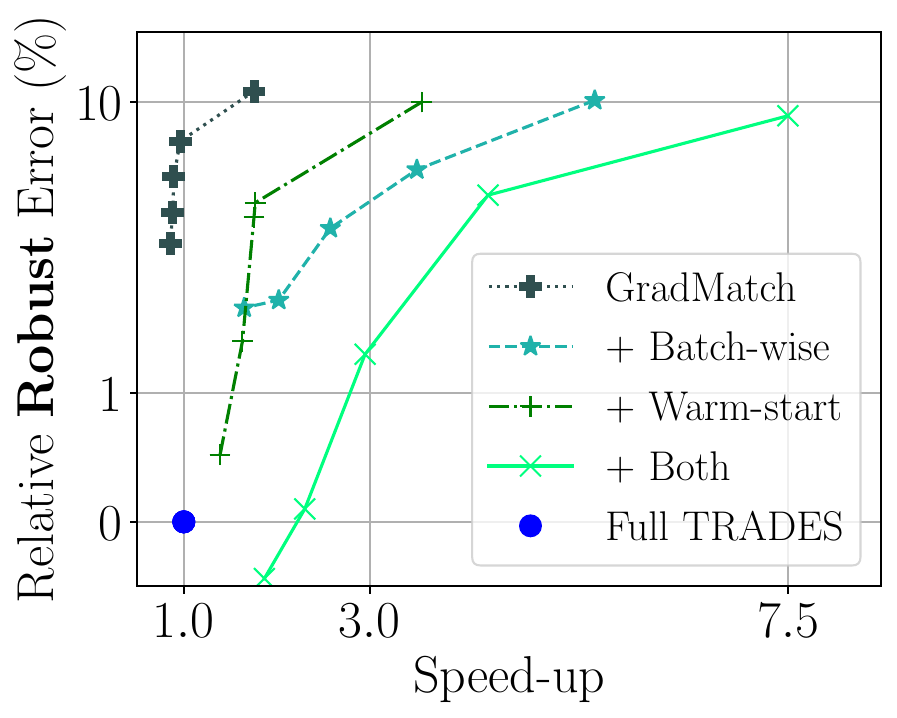}
		\caption{}
		\label{fig:tradeoff:gradmatch_robust}
	\end{subfigure}\hspace*{0.2em}
	\caption{\textcolor{revision}{Relative error vs. speed up curves for different versions of adversarial coreset selection in training CIFAR-10 models using the TRADES objective. In each curve, the coreset size is changed from 50\% to 10\% (left to right). (a, b) Clean and robust error vs.~speed up compared to full TRADES for different versions of Adversarial \textsc{Craig}. (c, d) Clean and robust error vs.~speed up compared to full TRADES for different versions of Adversarial \textsc{GradMatch}.}}
	\label{fig:tradeoff}
	\vskip -0.05 in
\end{figure*}

\section{Conclusion}\label{sec:conclusion}
In this paper, we proposed a general yet principled approach for efficient adversarial training based on the theory of coreset selection.
We discussed how repetitive computation of adversarial attacks for the entire training data could impede the training speed.
Unlike previous works that try to solve this issue by making the adversarial attack more straightforward, here, we took an orthogonal path to reduce the training set size without modifying the attacker. 
We first provided convergence bounds for adversarial training using a subset of the training data.
Our analysis showed that the convergence bound is related to how well this selected subset can approximate the loss gradient computed with the entire data.
Based on this study, we proposed to use the gradient approximation error as our coreset selection objective and tried to make a connection with vanilla coreset selection.
To this end, we discussed how coreset selection could be viewed as a two-step process, where first, the gradients for the entire training data are computed.
Then greedy solvers choose a weighted subset of data that can approximate the full gradient.
Using Danskin's theorem, we drew a connection between greedy coreset selection algorithms and adversarial training.
We then showed the flexibility of our adversarial coreset selection method by utilizing it for TRADES, $\ell_p$-PGD, and Perceptual Adversarial Training.
Our experimental results indicate that adversarial coreset selection can reduce the training time by more than 2-3 while slightly reducing the clean and robust accuracy.

\backmatter
\bmhead{Acknowledgments}
This research was undertaken using the LIEF HPC-GPG\-PU Facility hosted at the University of Melbourne.
This Facility was established with the assistance of LIEF Grant LE170100200.
Sarah Erfani is in part supported by Australian Research Council~(ARC) Discovery Early Career Researcher Award~(DECRA)~DE220100680.
Moreover, this research was partially supported by the ARC Centre of Excellence for Automated Decision-Making and Society~(CE200100005), and funded partially by the Australian Government through the Australian Research Council.

\newpage
\begin{appendices}
\onecolumn
\section{Greedy Subset Selection Algorithms}\label{ap:sec:greedy_selection}
	
This section briefly reviews the technical details of greedy subset selection algorithms used in our experiments.
Further details can be found in \citep{mirzasoleiman2020craig,killamsetty2021gradmatch}.

\subsection{\textsc{CRAIG}}\label{ap:sec:sec:craig}
As discussed previously, the goal of coreset selection is to find a subset of the training data such that the weighted gradient computed over this subset can give a good approximation to the full gradient.
Thus, \textsc{Craig}~\citep{mirzasoleiman2020craig} starts with explicitly writing down this objective as:
\begin{equation}\label{eq:gradient_app_craig}
\norm{\sum_{i \in \mathcal{V}} \nabla_{\boldsymbol{\theta}}\boldsymbol{\Phi} \left(f_{\boldsymbol{\theta}}; \boldsymbol{x}_{i}, y_{i}\right)-\sum_{j \in \mathcal{S}} \gamma_{j}\nabla_{\boldsymbol{\theta}}\boldsymbol{\Phi} \left(f_{\boldsymbol{\theta}}; \boldsymbol{x}_{j}, y_{j}\right)}.
\end{equation}
Here, $\mathcal{V}=\left[n\right]=\left\{1, 2, \dots, n\right\}$ denotes the training set.
The goal is to find a coreset $S \subseteq \mathcal{V}$ and its associated weights $\gamma_{j}$ such that the objective of \Cref{eq:gradient_app_craig} is minimized.
To this end, \citet{mirzasoleiman2020craig} find an upper-bound on the estimation error of \Cref{eq:gradient_app_craig}.
This way, it is shown that the coreset selection objective can be approximated by:
\begin{equation}\label{eq:craig_objective}
\mathcal{S}^{*}=\argmin_{\mathcal{S} \subseteq \mathcal{V}}\lvert\mathcal{S}\rvert, ~ \text {s.t.} ~ L(\mathcal{S}) \coloneqq \sum_{i \in \mathcal{V}} \min _{j \in \mathcal{S}} d_{i j} \leq \epsilon,
\end{equation}
where
\begin{equation}\label{eq:craig_dij}
d_{i j} \coloneqq \max _{\boldsymbol{\theta} \in \boldsymbol{\Theta}} \norm{\nabla_{\boldsymbol{\theta}}\boldsymbol{\Phi} \left(f_{\boldsymbol{\theta}}; \boldsymbol{x}_{i}, y_{i}\right)-\nabla_{\boldsymbol{\theta}}\boldsymbol{\Phi} \left(f_{\boldsymbol{\theta}}; \boldsymbol{x}_{j}, y_{j}\right)}
\end{equation}
denotes the maximum pairwise gradient distances computed for all $i \in \mathcal{V}$ and $j \in S$.
Then, \citet{mirzasoleiman2020craig} cast \Cref{eq:craig_objective} as the well-known \textit{submodular set cover problem} for which greedy solvers exist~\citep{minoux1978accelerated,nemhauser1978analysis,wolsey1982greedy}.
\Cref{alg:craig} shows this greedy selection algorithm.

\subsection{\textsc{GradMatch}}\label{ap:sec:sec:gradmatch}
\citet{killamsetty2021gradmatch} studies the convergence of adaptive data subset selection algorithms using \textit{stochastic gradient descent}~(SGD).
It is shown that the convergence bound consists of two terms: an irreducible noise-related term, and an additional gradient error term just like \Cref{eq:gradient_app_craig}.
Based on this analysis, \citet{killamsetty2021gradmatch} proposes to minimize this objective directly.
To this end, they use the famous Orthogonal Matching Pursuit~(OMP)~\citep{pati1992omp, elenberg2016restricted} as their greedy solver, resulting in an algorithm called \textsc{Grad-Match}.
It is then proved that since \textsc{GradMatch} minimizes the \Cref{eq:gradient_app_craig} objective directly, it achieves a lower error compared to \textsc{Craig} that only minimizes an upper-bound to \Cref{eq:gradient_app_craig}.
This algorithm is shown in \Cref{alg:gradmatch}.

\begin{algorithm*}[tb!]
	\caption{\textsc{Craig}~\citep{mirzasoleiman2020craig} \label{alg:craig}} 
	\textbf{Input}: Dataset $\{\boldsymbol{x}_{i}, y_{i}\}_{i=1}^{n}$ with indices $\mathcal{V}=\{1, 2, \dots, n\}$, the learning objective $\boldsymbol{\Phi}$ over classifier $f_{\boldsymbol{\theta}}$ with parameter set $\boldsymbol{\theta}.$
	\\
	\textbf{Output}: The subset of data instances $\mathcal{S} \subseteq \mathcal{V}$ with their corresponding weights $\{\gamma_{j}\}_{j \in \mathcal{S}}.$
	\\
	\textbf{Parameters}: Tolerance~$\varepsilon$. 
	\begin{algorithmic}[1]
	    \setstretch{1.4}
		\State Initialization: $\mathcal{S} \leftarrow \emptyset$, $s_{0} = 0$.
		\State Define: $F(\mathcal{S}) \coloneqq L(\{s_{0}\}) - L(\mathcal{S} \cup \{s_{0}\})$.
		\While {$F(\mathcal{S}) < L(\{s_{0}\}) - \varepsilon$}
			\State $j \in \argmax_{e \in \mathcal{V} \backslash \mathcal{S}} F\left(\mathcal{S} \cup \left\{e\right\}\right) - F\left(\mathcal{S}\right)$.
			\State $\mathcal{S} = \mathcal{S} \cup \{j\}$.
		\EndWhile
		\For {$j = 1$ to $\lvert\mathcal{S}\rvert$}
		    \State $\gamma_{j} = \sum_{i \in \mathcal{V}} \mathbb{I}\left[j = \argmin_{s \in \mathcal{S}} \max_{\boldsymbol{\theta} \in \boldsymbol{\Theta}}\norm{\nabla\boldsymbol{\Phi} \left(f_{\boldsymbol{\theta}}; \boldsymbol{x}_{i}, y_{i}\right)-\nabla\boldsymbol{\Phi} \left(f_{\boldsymbol{\theta}}; \boldsymbol{x}_{j}, y_{j}\right)}\right]$.
		\EndFor
		\State \Return $\mathcal{S}$, $\boldsymbol{\gamma}$.
	\end{algorithmic}
\end{algorithm*}

\begin{algorithm*}[tb!]
	\caption{OMP selection algorithm used in \textsc{GradMatch}~\citep{killamsetty2021gradmatch} \label{alg:gradmatch}} 
	\textbf{Input}: Dataset $\{\boldsymbol{x}_{i}, y_{i}\}_{i=1}^{n}$ with indices $\mathcal{V}=\{1, 2, \dots, n\}$, the learning objective $\boldsymbol{\Phi}$ over classifier $f_{\boldsymbol{\theta}}$ with parameter set $\boldsymbol{\theta}.$
	\\
	\textbf{Output}: The subset of data instances $\mathcal{S} \subseteq \mathcal{V}$ with their corresponding weights $\{\gamma_{j}\}_{j \in \mathcal{S}}.$
	\\
	\textbf{Parameters}: Tolerance~$\varepsilon$, regularization coefficient~$\lambda$, coreset size~$k$. 
	\begin{algorithmic}[1]
	    \setstretch{1.4}
	    \State Initialization: $\mathcal{S} \leftarrow \emptyset$, $\bar{\boldsymbol{\gamma}} \leftarrow \boldsymbol{0}_{n}$.
	    \State Define: $\Gamma \left(\mathcal{V}, \bar{\boldsymbol{\gamma}}, \boldsymbol{\theta}\right) = \left\|\sum_{i \in \mathcal{V}} \nabla\boldsymbol{\Phi}\left(f_{\boldsymbol{\theta}}; \boldsymbol{x}_{i}, y_{i}\right)-\sum_{j \in \mathcal{V}} \bar{\gamma}_{j} \nabla \boldsymbol{\Phi}\left(f_{\boldsymbol{\theta}}; \boldsymbol{x}_{j}, y_{j}\right)\right\| + \lambda \norm{\bar{\boldsymbol{\gamma}}}^{2}$.
		\State Set: $\boldsymbol{r} \leftarrow \nabla_{\bar{\boldsymbol{\gamma}}}{\Gamma \left(\mathcal{V}, \bar{\boldsymbol{\gamma}}, \boldsymbol{\theta}\right)}\rvert_{\bar{\boldsymbol{\gamma}} = \boldsymbol{0}}$.
		\While {$\lvert\mathcal{S}\rvert \leq k$ and  $\Gamma \left(\mathcal{V}, \bar{\boldsymbol{\gamma}}, \boldsymbol{\theta}\right) \geq \varepsilon$}
			\State $j \in \argmax_{i \in \mathcal{V}}  \lvert r_{i} \rvert$.
			\State $\mathcal{S} = \mathcal{S} \cup \{j\}$.
			\State $\bar{\boldsymbol{\gamma}}[\mathcal{S}] \leftarrow \argmin_{\bar{\boldsymbol{\gamma}}: \mathrm{supp}(\bar{\boldsymbol{\gamma}}) \subseteq \mathcal{S}} \Gamma \left(\mathcal{V}, \bar{\boldsymbol{\gamma}}, \boldsymbol{\theta}\right)$.
			\State $\boldsymbol{r} \leftarrow \nabla_{\bar{\boldsymbol{\gamma}}}{\Gamma \left(\mathcal{V}, \bar{\boldsymbol{\gamma}}, \boldsymbol{\theta}\right)}$.
		\EndWhile
		\State $\boldsymbol{\gamma} = \bar{\boldsymbol{\gamma}}[\mathcal{S}]$.
		\State \Return $\mathcal{S}$, $\boldsymbol{\gamma}$.
	\end{algorithmic}
\end{algorithm*}

\section{Proofs}\label{ap:proofs}

\subsection{TRADES Gradient}\label{ap:trades_gradient}
To compute the \textit{second} gradient term in \Cref{eq:nn_gradeint_trades_penufinal} let us assume that $\boldsymbol{w}(\boldsymbol{\theta}) = f_{\boldsymbol{\theta}}(\boldsymbol{x}_{\rm{adv}})$ and ${\boldsymbol{z}(\boldsymbol{\theta}) = f_{\boldsymbol{\theta}}(\boldsymbol{x})}$.
We can write the aforementioned gradient as:
\begin{small}
		\begin{align}\nonumber
		&\nabla_{\boldsymbol{\theta}} \mathcal{L}_{\mathrm{CE}}\left(f_{\boldsymbol{\theta}}(\boldsymbol{x}_{\rm{adv}}), f_{\boldsymbol{\theta}}(\boldsymbol{x})\right)\\\nonumber 
		&\qquad = \nabla_{\boldsymbol{\theta}} \mathcal{L}_{\mathrm{CE}}\left(\boldsymbol{w}(\boldsymbol{\theta}), \boldsymbol{z}(\boldsymbol{\theta})\right)\\\nonumber
		&\qquad\stackrel{(1)}{=} \tfrac{\partial \mathcal{L}_{\mathrm{CE}}}{\partial \boldsymbol{w}} \nabla_{\boldsymbol{\theta}}\boldsymbol{w}(\boldsymbol{\theta}) + \tfrac{\partial \mathcal{L}_{\mathrm{CE}}}{\partial \boldsymbol{z}} \nabla_{\boldsymbol{\theta}}\boldsymbol{z}(\boldsymbol{\theta})\\\nonumber
		&\qquad\stackrel{(2)}{=} \nabla_{\boldsymbol{\theta}} \mathcal{L}_{\mathrm{CE}}\left(f_{\boldsymbol{\theta}}(\boldsymbol{x}_{\rm{adv}}), \text{\texttt{freeze}}\left(f_{\boldsymbol{\theta}}(\boldsymbol{x})\right)\right) \\ \label{eq:trades_chain_rule_ap}
		&\qquad\quad + \nabla_{\boldsymbol{\theta}} \mathcal{L}_{\mathrm{CE}}\left(\text{\texttt{freeze}}\left(f_{\boldsymbol{\theta}}(\boldsymbol{x}_{\rm{adv}})\right), f_{\boldsymbol{\theta}}(\boldsymbol{x})\right).
		\end{align}
\end{small}

\noindent Here, step (1) is derived using the multi-variable chain rule. 
Also, step (2) is the re-writing of step (1) by using the $\text{\texttt{freeze}}(\cdot)$ kernel that stops the gradients from backpropagating through its argument function.

\subsection{Convergence Proofs}\label{ap:convergence_proofs}
First, we establish the relationship between the Lipschitzness and strong convexity of the (loss) function $\mathcal{L}$ and its max-function $\max\mathcal{L}$.
To this end, we prove the following two lemmas.
These results will be an essential part of the main convergence theorem proof.
Note that our results in the section follow our notations established in \Cref{sec:sec:convergence}.

\begin{lemma}\label{lipschitz_lemma}
    Let $\mathcal{K}$ be a nonempty compact topological space, $\mathcal{L}: \mathbb{R}^{m} \times \mathcal{K} \rightarrow \mathbb{R}$ be such that $\mathcal{L}(\cdot, \boldsymbol{\delta})$ is Lipschitz continuous for every $\boldsymbol{\delta} \in \mathcal{K}$.
	Then, the corresponding max-function
	$$
	\phi(\boldsymbol{\theta})=\max _{\delta \in \mathcal{K}} \mathcal{L}(\boldsymbol{\theta}, \boldsymbol{\delta})
	$$
	is also Lipschitz continuous.
\end{lemma}
\begin{proof}
    Since $\mathcal{L}(\boldsymbol{\theta}, \boldsymbol{\delta})$ is assumed to be Lipschitz for every $\boldsymbol{\delta} \in \mathcal{K}$, then for any input pair $\boldsymbol{\theta}_1, \boldsymbol{\theta}_2 \in \mathbb{R}^{m}$ we can write:
    \begin{equation}\label{eq:lipschitz}
        \mathcal{L}(\boldsymbol{\theta}_1, \boldsymbol{\delta}) \leq \mathcal{L}(\boldsymbol{\theta}_2, \boldsymbol{\delta}) + \sigma \norm{\boldsymbol{\theta}_1 - \boldsymbol{\theta}_2},
    \end{equation}
    where $\norm{\cdot}$ is a distance metric, and $\sigma$ is the Lipschitz constant.
    Now, by the definition of the max-function we can write:
    \begin{equation}\label{eq:max}
        \mathcal{L}(\boldsymbol{\theta}_2, \boldsymbol{\delta}) \leq \max _{\delta \in \mathcal{K}} \mathcal{L}(\boldsymbol{\theta}_2, \boldsymbol{\delta}) = \phi(\boldsymbol{\theta}_2).
    \end{equation}
    Plugging \Cref{eq:max} into \Cref{eq:lipschitz}, we then write:
    \begin{equation}\label{eq:lipschitz_max}
        \mathcal{L}(\boldsymbol{\theta}_1, \boldsymbol{\delta}) \leq \phi(\boldsymbol{\theta}_2) + \sigma \norm{\boldsymbol{\theta}_1 - \boldsymbol{\theta}_2}.
    \end{equation}
    Since \Cref{eq:lipschitz_max} holds for every $\boldsymbol{\delta} \in \mathcal{K}$, it means that $\phi(\boldsymbol{\theta}_1) = \max _{\delta \in \mathcal{K}} \mathcal{L}(\boldsymbol{\theta}_1, \boldsymbol{\delta})$ is also less than the RHS.
    In other words, we have:
    \begin{equation}
        \phi(\boldsymbol{\theta}_1) \leq \phi(\boldsymbol{\theta}_2) + \sigma \norm{\boldsymbol{\theta}_1- \boldsymbol{\theta}_2}.
    \end{equation}
    A similar inequality can be derived by switching $\boldsymbol{\theta}_1$ and $\boldsymbol{\theta}_2$.
    Thus, we can conclude:
    \begin{equation}
        \lvert\phi(\boldsymbol{\theta}_1)  - \phi(\boldsymbol{\theta}_2)\rvert \leq \sigma \norm{\boldsymbol{\theta}_1 - \boldsymbol{\theta}_2},
    \end{equation}
    and hence, $\phi(\boldsymbol{\theta})=\max _{\delta \in \mathcal{K}} \mathcal{L}(\boldsymbol{\theta}, \boldsymbol{\delta})$ is Lipschitz.
\end{proof}

\begin{lemma}\label{strongly_convex_lemma}
    Let $\mathcal{K}$ be a nonempty compact topological space, $\mathcal{L}: \mathbb{R}^{m} \times \mathcal{K} \rightarrow \mathbb{R}$ be such that $\mathcal{L}(\cdot, \boldsymbol{\delta})$ is $\mu$-strongly convex for every $\boldsymbol{\delta} \in \mathcal{K}$.
	Then, the corresponding max-function
	$$
	\phi(\boldsymbol{\theta})=\max _{\delta \in \mathcal{K}} \mathcal{L}(\boldsymbol{\theta}, \boldsymbol{\delta})
	$$
	is also $\mu$-strongly convex if for all $\boldsymbol{\theta} \in \mathbb{R}^{m}$ the set ${\boldsymbol{\delta}^{*}(\boldsymbol{\theta})=\left\{\boldsymbol{\delta} \in \arg \max _{\boldsymbol{\delta} \in \mathcal{K}} \mathcal{L}(\boldsymbol{\theta}, \boldsymbol{\delta})\right\}}$ is a singleton, i.e. $\boldsymbol{\delta}^{*}(\boldsymbol{\theta})=\left\{\boldsymbol{\delta}_{\boldsymbol{\theta}}^{*}\right\}$.
\end{lemma}
\begin{proof}
    $\mathcal{L}(\cdot, \boldsymbol{\delta})$ is $\mu$-strongly convex for every $\boldsymbol{\delta} \in \mathcal{K}$.
    From the definition of $\mu$-strongly convex functions we have:
    \begin{align}\label{mu_strongly_convex}
        {\nabla_{\boldsymbol{\theta}} \mathcal{L}\left(\boldsymbol{\theta}_2, \boldsymbol{\delta}\right)}^{\intercal}\left(\boldsymbol{\theta}_2 - \boldsymbol{\theta}_1\right) \geq \mathcal{L}\left(\boldsymbol{\theta}_2, \boldsymbol{\delta}\right) - \mathcal{L}\left(\boldsymbol{\theta}_1, \boldsymbol{\delta}\right) + \frac{\mu}{2}\norm{\boldsymbol{\theta}_2 - \boldsymbol{\theta}_1}^2.
    \end{align}
    By the definition of the max function, we know that $\mathcal{L}(\boldsymbol{\theta}_1, \boldsymbol{\delta}) \leq \max _{\delta \in \mathcal{K}} \mathcal{L}(\boldsymbol{\theta}_1, \boldsymbol{\delta}) = \phi(\boldsymbol{\theta}_1)$.
    Hence, we can replace this in \Cref{mu_strongly_convex} to get:
    \begin{align}\label{mu_strongly_convex_theta1}
        {\nabla_{\boldsymbol{\theta}} \mathcal{L}\left(\boldsymbol{\theta}_2, \boldsymbol{\delta}\right)}^{\intercal}\left(\boldsymbol{\theta}_2 - \boldsymbol{\theta}_1\right) \geq \mathcal{L}\left(\boldsymbol{\theta}_2, \boldsymbol{\delta}\right) - \phi(\boldsymbol{\theta}_1)  + \frac{\mu}{2}\norm{\boldsymbol{\theta}_2 - \boldsymbol{\theta}_1}^2.
    \end{align}
    Since \Cref{mu_strongly_convex_theta1} is set to hold for all $\boldsymbol{\delta} \in \mathcal{K}$, it also holds for $\boldsymbol{\delta}^{*}(\boldsymbol{\theta}_2)=\left\{\boldsymbol{\delta}_{\boldsymbol{\theta}_2}^{*}\right\}$.
    Setting $\boldsymbol{\delta} = \boldsymbol{\delta}_{\boldsymbol{\theta}_2}^{*}$ we get:
    \begin{align}\label{mu_strongly_convex_theta2}
        {\nabla_{\boldsymbol{\theta}} \mathcal{L}\left(\boldsymbol{\theta}_2, \boldsymbol{\delta}_{\boldsymbol{\theta}_2}^{*}\right)}^{\intercal}\left(\boldsymbol{\theta}_2 - \boldsymbol{\theta}_1\right) \geq \mathcal{L}\left(\boldsymbol{\theta}_2, \boldsymbol{\delta}_{\boldsymbol{\theta}_2}^{*}\right) - \phi(\boldsymbol{\theta}_1) + \frac{\mu}{2}\norm{\boldsymbol{\theta}_2 - \boldsymbol{\theta}_1}^2.
    \end{align}
    From the Danskin's theorem we know that ${\nabla_{\boldsymbol{\theta}}\phi(\boldsymbol{\theta}_2) = {\nabla_{\boldsymbol{\theta}} \mathcal{L}\left(\boldsymbol{\theta}_2, \boldsymbol{\delta}_{\boldsymbol{\theta}_2}^{*}\right)}}$.
    Also, since $\boldsymbol{\delta}_{\boldsymbol{\theta}_2}^{*}$ is the maximizer of $\mathcal{L}\left(\boldsymbol{\theta}_2, \boldsymbol{\delta}\right)$, thus $\mathcal{L}(\boldsymbol{\theta}_2, \boldsymbol{\delta}_{\boldsymbol{\theta}_2}^{*}) = \phi(\boldsymbol{\theta}_2)$.
    Replacing these terms in \Cref{mu_strongly_convex_theta2} we arrive at:
    \begin{equation}\label{mu_strongly_convex_gradient}
        {\nabla_{\boldsymbol{\theta}}\phi(\boldsymbol{\theta}_2)}^{\intercal}\left(\boldsymbol{\theta}_2 - \boldsymbol{\theta}_1\right) \geq \phi(\boldsymbol{\theta}_2) - \phi(\boldsymbol{\theta}_1) + \frac{\mu}{2}\norm{\boldsymbol{\theta}_2 - \boldsymbol{\theta}_1}^2,
    \end{equation}
    which means that $\phi(\boldsymbol{\theta})=\max _{\delta \in \mathcal{K}} \mathcal{L}(\boldsymbol{\theta}, \boldsymbol{\delta})$ is also $\mu$-strongly convex.
\end{proof}

Now that we have established our preliminary lemmas, we are ready to re-state our main theorem and prove it.

\begin{customthm}{1}[restated]\label{thm:convergence:rep}
Let $\boldsymbol{\gamma}^{t}$ and $\mathcal{S}^{t}$ denote the weights and subset derived by any \textbf{adversarial coreset selection} algorithm at iteration $t$ of the full gradient descent.
Also, let $\boldsymbol{\theta}^{*}$ be the optimal model parameters, $\mathcal{L}$ be a convex loss function with respect to $\boldsymbol{\theta}$, and that the parameters are bounded such that $\norm{\boldsymbol{\theta} - \boldsymbol{\theta}^{*}} \leq \Delta$.
Moreover, let us define the gradient approximation error at iteration $t$ with:
\begin{equation}\nonumber
    \Gamma(L, L_{\boldsymbol{\gamma}}, \boldsymbol{\gamma}^{t}, \mathcal{S}^{t}, \boldsymbol{\theta}_t) \coloneqq \norm{\nabla_{\boldsymbol{\theta}}L(\boldsymbol{\theta}_t) - \nabla_{\boldsymbol{\theta}}L^{\mathcal{S}^{t}}_{\boldsymbol{\gamma}^{t}}{(\boldsymbol{\theta}_t)}}.
\end{equation}

Then, for $t=0, 1, \cdots, T-1$ the following guarantees hold:

(1) For a Lipschitz continuous loss function $\mathcal{L}$ with parameter $\sigma$ and constant learning rate ${\alpha=\frac{\Delta}{\sigma \sqrt{T}}}$ we have:
\begin{align}\nonumber
    \min _{t=0: T-1} L(\boldsymbol{\theta}_{t})-L(\boldsymbol{\theta}^{*}) \leq \frac{\Delta \sigma}{\sqrt{T}} +\frac{\Delta}{T} \sum_{t=0}^{T-1} \Gamma(L, L_{\boldsymbol{\gamma}^t}, \boldsymbol{\gamma}^{t}, \mathcal{S}^{t}, \boldsymbol{\theta}_t).
\end{align}

(2) Moreover, for a Lipschitz continuous loss $\mathcal{L}$ with parameter $\sigma$ and strongly convex with parameter $\mu$, by setting a learning rate $\alpha_{t}=\frac{2}{n\mu(1+t)}$ we have:
\begin{align}\nonumber
\min _{t=0: T-1} L(\boldsymbol{\theta}_{t})-L(\boldsymbol{\theta}^{*}) \leq \frac{2 \sigma^{2}}{n\mu(T-1)} +\sum_{t=0}^{T-1} \frac{2 \Delta t}{T(T-1)} \Gamma(L, L_{\boldsymbol{\gamma}^t}, \boldsymbol{\gamma}^{t}, \mathcal{S}^{t}, \boldsymbol{\theta}_t),
\end{align}
where $n$ is the total number of training data.
\end{customthm}

\begin{proof}
We take a similar approach to that of \citet{killamsetty2021gradmatch} to prove these convergence guarantees.
In particular, we first derive a general relationship that relates the gradients to the optimal model parameters and then use the conditions of each part to simplify and get to the final result.

Using the gradient descent update rule in \Cref{eq:gradient_app_coreset}, we can re-write it as:
\begin{equation}\label{eq:gd_alternative}
    {\nabla_{\boldsymbol{\theta}}L_{\boldsymbol{\gamma}^{t}}(\boldsymbol{\theta}_{t})}^{\intercal}\left(\boldsymbol{\theta}_{t} - \boldsymbol{\theta}^{*}\right)
    = \frac{1}{\alpha_{t}}{\left(\boldsymbol{\theta}_{t} - \boldsymbol{\theta}_{t + 1}\right)}^{\intercal}\left(\boldsymbol{\theta}_{t} - \boldsymbol{\theta}^{*}\right).
\end{equation}
Using the identity $2 \boldsymbol{a}^{\intercal}\boldsymbol{b} = \norm{\boldsymbol{a}}^2 + \norm{\boldsymbol{b}}^2 - \norm{\boldsymbol{a} - \boldsymbol{b}}^2$, we can simplify \Cref{eq:gd_alternative} with:
\begin{align}\label{eq:gd_alternative_simplified}\nonumber
    &{\nabla_{\boldsymbol{\theta}}L_{\boldsymbol{\gamma}^{t}}(\boldsymbol{\theta}_{t})}^{\intercal}\left(\boldsymbol{\theta}_{t} - \boldsymbol{\theta}^{*}\right)\\\nonumber
    &= \frac{\norm{\boldsymbol{\theta}_{t} - \boldsymbol{\theta}_{t + 1}}^{2} + \norm{\boldsymbol{\theta}_{t} - \boldsymbol{\theta}^{*}}^{2} - \norm{\boldsymbol{\theta}_{t+1} - \boldsymbol{\theta}^{*}}^{2}}{2\alpha_{t}}\\
    &= \frac{\norm{\alpha_{t} \nabla_{\boldsymbol{\theta}}L_{\boldsymbol{\gamma}^{t}}(\boldsymbol{\theta}_{t})}^{2} + \norm{\boldsymbol{\theta}_{t} - \boldsymbol{\theta}^{*}}^{2} - \norm{\boldsymbol{\theta}_{t+1} - \boldsymbol{\theta}^{*}}^{2}}{2\alpha_{t}},
\end{align}
where we replaced $\boldsymbol{\theta}_{t} - \boldsymbol{\theta}_{t + 1}$ with $\alpha_{t} \nabla_{\boldsymbol{\theta}}L_{\boldsymbol{\gamma}^{t}}(\boldsymbol{\theta}_{t})$ using the gradient descent update rule from \Cref{eq:gradient_descent_coreset}.
Now, we can re-write the LHS of \Cref{eq:gd_alternative_simplified} as:
\begin{align}\nonumber
    &{\nabla_{\boldsymbol{\theta}}L_{\boldsymbol{\gamma}^{t}}(\boldsymbol{\theta}_{t})}^{\intercal}\left(\boldsymbol{\theta}_{t} - \boldsymbol{\theta}^{*}\right)\\\nonumber
    &\qquad= {\nabla_{\boldsymbol{\theta}}L_{\boldsymbol{\gamma}^{t}}(\boldsymbol{\theta}_{t})}^{\intercal}\left(\boldsymbol{\theta}_{t} - \boldsymbol{\theta}^{*}\right)\cdots \\\nonumber
    \quad &\cdots- {\nabla_{\boldsymbol{\theta}}L(\boldsymbol{\theta}_{t})}^{\intercal}\left(\boldsymbol{\theta}_{t} - \boldsymbol{\theta}^{*}\right) + {\nabla_{\boldsymbol{\theta}}L(\boldsymbol{\theta}_{t})}^{\intercal}\left(\boldsymbol{\theta}_{t} - \boldsymbol{\theta}^{*}\right),
\end{align}
by adding and subtracting the full gradient ${\nabla_{\boldsymbol{\theta}}L(\boldsymbol{\theta}_{t})}$.
Keeping ${\nabla_{\boldsymbol{\theta}}L(\boldsymbol{\theta}_{t})}^{\intercal}\left(\boldsymbol{\theta}_{t} - \boldsymbol{\theta}^{*}\right)$ on the LHS, we move the remaining two terms to the RHS of \Cref{eq:gd_alternative_simplified}.
Thus, we get:
\begin{align}\nonumber\label{eq:gd_base}
    &{\nabla_{\boldsymbol{\theta}}L(\boldsymbol{\theta}_{t})}^{\intercal}\left(\boldsymbol{\theta}_{t} - \boldsymbol{\theta}^{*}\right)\\\nonumber
    &=\frac{\norm{\alpha_{t} \nabla_{\boldsymbol{\theta}}L_{\boldsymbol{\gamma}^{t}}(\boldsymbol{\theta}_{t})}^{2} + \norm{\boldsymbol{\theta}_{t} - \boldsymbol{\theta}^{*}}^{2} - \norm{\boldsymbol{\theta}_{t+1} - \boldsymbol{\theta}^{*}}^{2}}{2\alpha_{t}} \cdots\\
    &\quad\cdots + \left(\nabla_{\boldsymbol{\theta}}L(\boldsymbol{\theta}_{t}) - \nabla_{\boldsymbol{\theta}}L_{\boldsymbol{\gamma}^{t}}(\boldsymbol{\theta}_{t}) \right)^{\intercal}\left(\boldsymbol{\theta}_{t} - \boldsymbol{\theta}^{*}\right).
\end{align}

Assuming a constant learning rate $\alpha_t = \alpha$, we can sum both sides of \Cref{eq:gd_base} for all values ${t=0, 1, \dots, T-1}$.
We then get:
\begin{align}\nonumber
    \sum_{t=0}^{T-1}&{\nabla_{\boldsymbol{\theta}}L(\boldsymbol{\theta}_{t})}^{\intercal}\left(\boldsymbol{\theta}_{t} - \boldsymbol{\theta}^{*}\right)\\\nonumber
    &= \frac{\norm{\boldsymbol{\theta}_{0} - \boldsymbol{\theta}^{*}}^{2} - \norm{\boldsymbol{\theta}_{T} - \boldsymbol{\theta}^{*}}^{2}}{2\alpha}\cdots\\\nonumber
    \quad\cdots&+ \sum_{t=0}^{T-1}\frac{\alpha}{2}\norm{ \nabla_{\boldsymbol{\theta}}L_{\boldsymbol{\gamma}^{t}}(\boldsymbol{\theta}_{t})}^{2}\\
    \quad\cdots&+ \sum_{t=0}^{T-1}\left(\nabla_{\boldsymbol{\theta}}L(\boldsymbol{\theta}_{t}) - \nabla_{\boldsymbol{\theta}}L_{\boldsymbol{\gamma}^{t}}(\boldsymbol{\theta}_{t}) \right)^{\intercal}\left(\boldsymbol{\theta}_{t} - \boldsymbol{\theta}^{*}\right),
\end{align}
where the first two terms on the RHS are due to summation of a telescoping series.
Using the $\ell_2$ norm properties, we know that ${\norm{\boldsymbol{\theta}_{T} - \boldsymbol{\theta}^{*}}^{2} \geq 0}$.
Hence, we can drop this term and replace the equation with an inequality to get
\begin{align}\nonumber\label{eq:gd_sum_constant_lr}
    \sum_{t=0}^{T-1}{\nabla_{\boldsymbol{\theta}}L(\boldsymbol{\theta}_{t})}^{\intercal}\left(\boldsymbol{\theta}_{t} - \boldsymbol{\theta}^{*}\right)
    &\leq \frac{\norm{\boldsymbol{\theta}_{0} - \boldsymbol{\theta}^{*}}^{2}}{2\alpha}
    + \sum_{t=0}^{T-1}\frac{\alpha}{2}\norm{ \nabla_{\boldsymbol{\theta}}L_{\boldsymbol{\gamma}^{t}}(\boldsymbol{\theta}_{t})}^{2} \cdots\\
    \quad\cdots&+ \sum_{t=0}^{T-1}\left(\nabla_{\boldsymbol{\theta}}L(\boldsymbol{\theta}_{t}) - \nabla_{\boldsymbol{\theta}}L_{\boldsymbol{\gamma}^{t}}(\boldsymbol{\theta}_{t}) \right)^{\intercal}\left(\boldsymbol{\theta}_{t} - \boldsymbol{\theta}^{*}\right).
\end{align}
Until this point, our assumptions are pretty general about the loss function $\mathcal{L}$.
Next, we are going to derive convergence bounds for different choices of the loss function $\mathcal{L}$.

\bmhead{Convex and Lipschitz Continuous Loss Function $\mathcal{L}$}
As the first item, let $\mathcal{L}$ be a continuous loss function that is both convex and Lipschitz continuous with constant $\sigma$.

From the Danskin's theorem~(\Cref{danskin_theorem}), we know that if $\mathcal{L}(\boldsymbol{\theta}; \tilde{\boldsymbol{x}})$ is convex in $\boldsymbol{\theta}$ for every $\tilde{\boldsymbol{x}}$, then $\max_{\tilde{\boldsymbol{x}}} \mathcal{L}(\boldsymbol{\theta}; \tilde{\boldsymbol{x}})$ is also going to be convex.
As such, using \Cref{eq:emp_loss} we can conclude that $L(\boldsymbol{\theta})$ is also convex.
This is true since $L(\boldsymbol{\theta})$ is a summation of convex functions.
Because $L(\boldsymbol{\theta})$ is convex, we can then write:
\begin{equation}\label{convex_L}
    L(\boldsymbol{\theta}_t) - L(\boldsymbol{\theta}^{*}) \leq {\nabla_{\boldsymbol{\theta}}L(\boldsymbol{\theta}_{t})}^{\intercal}\left(\boldsymbol{\theta}_{t} - \boldsymbol{\theta}^{*}\right).
\end{equation}
Considering \Cref{eq:gd_sum_constant_lr,convex_L}, we next conclude that:
\begin{align}\label{eq:gd_convex_L}
    \sum_{t=0}^{T-1}L(\boldsymbol{\theta}_t) - L(\boldsymbol{\theta}^{*})
    &\leq \frac{\norm{\boldsymbol{\theta}_{0} - \boldsymbol{\theta}^{*}}^{2}}{2\alpha}
    + \sum_{t=0}^{T-1}\frac{\alpha}{2}\norm{ \nabla_{\boldsymbol{\theta}}L_{\boldsymbol{\gamma}^{t}}(\boldsymbol{\theta}_{t})}^{2}\cdots\\
    \cdots&+ \sum_{t=0}^{T-1}\left(\nabla_{\boldsymbol{\theta}}L(\boldsymbol{\theta}_{t}) - \nabla_{\boldsymbol{\theta}}L_{\boldsymbol{\gamma}^{t}}(\boldsymbol{\theta}_{t}) \right)^{\intercal}\left(\boldsymbol{\theta}_{t} - \boldsymbol{\theta}^{*}\right).
\end{align}
From \Cref{lipschitz_lemma} we know that if $\mathcal{L}(\boldsymbol{\theta}; \tilde{\boldsymbol{x}})$ is Lipschitz with constant $\sigma$, then so does $\max_{\tilde{\boldsymbol{x}}} \mathcal{L}(\boldsymbol{\theta}; \tilde{\boldsymbol{x}})$.
Moreover, by Danskin's theorem we have that $\max_{\tilde{\boldsymbol{x}}} \mathcal{L}(\boldsymbol{\theta}; \tilde{\boldsymbol{x}})$ is both convex and differentiable.
Hence, we can conclude that:
\begin{equation}\label{eq:lipschitz_gradient}
    \norm{\nabla_{\boldsymbol{\theta}} \max_{\tilde{\boldsymbol{x}}} \mathcal{L}(\boldsymbol{\theta}; \tilde{\boldsymbol{x}})} \leq \sigma.
\end{equation}
Using this, we can conclude from \Cref{eq:emp_loss_coreset} that:
\begin{align}\nonumber\label{eq:coreset_gradient_lipschitz}
        \norm{\nabla_{\boldsymbol{\theta}} L_{\boldsymbol{\gamma}^{t}}{(\boldsymbol{\theta}_{t})}} &= \norm{\sum_{j \in \mathcal{S}^{t}} \gamma^{t}_{j} \nabla_{\boldsymbol{\theta}} \max_{\tilde{\boldsymbol{x}}_j} \mathcal{L}(\boldsymbol{\theta}_{t}; \tilde{\boldsymbol{x}}_j)}\\\nonumber
        &\stackrel{(1)}{\leq} \sum_{j \in \mathcal{S}^{t}} \gamma^{t}_{j} \norm{\nabla_{\boldsymbol{\theta}} \max_{\tilde{\boldsymbol{x}}_j} \mathcal{L}(\boldsymbol{\theta}_{t}; \tilde{\boldsymbol{x}}_j)}\\
        &\stackrel{(2)}{\leq} \sum_{j \in \mathcal{S}^{t}} \gamma^{t}_{j} \sigma,
\end{align}
where (1) is the result of the triangle inequality, and (2) is derived using \Cref{eq:lipschitz_gradient}.
We further assume that the gradients as well as the weights are normalized such that $\sum_{j \in \mathcal{S}^{t}} \gamma^{t}_{j} = 1$.
Hence, we can plug \Cref{eq:coreset_gradient_lipschitz} into \Cref{eq:gd_convex_L} to get:
\begin{align}\nonumber\label{eq:gd_convex_lipschitz_L}
    \sum_{t=0}^{T-1}L(\boldsymbol{\theta}_t) - L(\boldsymbol{\theta}^{*})&\stackrel{(1)}{\leq} \frac{\norm{\boldsymbol{\theta}_{0} - \boldsymbol{\theta}^{*}}^{2}}{2\alpha}
    + \frac{T \alpha \sigma^2}{2} + \sum_{t=0}^{T-1}\left(\nabla_{\boldsymbol{\theta}}L(\boldsymbol{\theta}_{t}) - \nabla_{\boldsymbol{\theta}}L_{\boldsymbol{\gamma}^{t}}(\boldsymbol{\theta}_{t}) \right)^{\intercal}\left(\boldsymbol{\theta}_{t} - \boldsymbol{\theta}^{*}\right)\\
    &\stackrel{(2)}{\leq} \frac{\norm{\boldsymbol{\theta}_{0} - \boldsymbol{\theta}^{*}}^{2}}{2\alpha}
    + \frac{T \alpha \sigma^2}{2} + \sum_{t=0}^{T-1}\norm{\nabla_{\boldsymbol{\theta}}L(\boldsymbol{\theta}_{t}) - \nabla_{\boldsymbol{\theta}}L_{\boldsymbol{\gamma}^{t}}(\boldsymbol{\theta}_{t}) }\norm{\boldsymbol{\theta}_{t} - \boldsymbol{\theta}^{*}},
\end{align}
where (1) is deduced using \Cref{eq:coreset_gradient_lipschitz}, and (2) is derived using the Cauchy-Schwarz inequality.
From our assumptions we have $\norm{\boldsymbol{\theta} - \boldsymbol{\theta}^{*}} \leq \Delta$.
As such, we re-write \Cref{eq:gd_convex_lipschitz_L} using $\norm{\boldsymbol{\theta} - \boldsymbol{\theta}^{*}} \leq \Delta$ and dividing both sides by $T$:
\begin{align}\label{eq:gd_convex_lipschitz_L_2}
    \frac{1}{T}\sum_{t=0}^{T-1}L(\boldsymbol{\theta}_t) - L(\boldsymbol{\theta}^{*})\leq \frac{\Delta ^ 2}{2 \alpha T}
    + \frac{\alpha \sigma^2}{2} + \frac{\Delta}{T}\sum_{t=0}^{T-1}\norm{\nabla_{\boldsymbol{\theta}}L(\boldsymbol{\theta}_{t}) - \nabla_{\boldsymbol{\theta}}L_{\boldsymbol{\gamma}^{t}}(\boldsymbol{\theta}_{t})}.
\end{align}
Since
$${\min_{t=0: T-1} L(\boldsymbol{\theta}_{t})-L(\boldsymbol{\theta}^{*}) \leq \frac{1}{T}\sum_{t=0}^{T-1}L(\boldsymbol{\theta}_t) - L(\boldsymbol{\theta}^{*})},$$
we can get:
\begin{align}\label{eq:gd_convex_lipschitz_L_3}
    \min_{t=0: T-1} L(\boldsymbol{\theta}_{t})-L(\boldsymbol{\theta}^{*})
    \leq \frac{\Delta ^ 2}{2 \alpha T} + \frac{\alpha \sigma^2}{2} + \frac{\Delta}{T}\sum_{t=0}^{T-1}\Gamma(L, L_{\boldsymbol{\gamma}}, \boldsymbol{\gamma}^{t}, \mathcal{S}^{t}, \boldsymbol{\theta}_t).
\end{align}
And finally, we can show that $\alpha^2 = \frac{\Delta ^2}{\sigma ^ 2 T}$ minimizes the addition of the first two terms.
Hence, by plugging $\alpha = \frac{\Delta}{\sigma \sqrt{T}}$ we get our first result:
\begin{align}\label{eq:final_gd_convex_lipschitz_L}
    \min_{t=0: T-1} L(\boldsymbol{\theta}_{t}) -L(\boldsymbol{\theta}^{*})
    \leq \frac{\Delta \sigma}{\sqrt{T}} + \frac{\Delta}{T}\sum_{t=0}^{T-1}\Gamma(L, L_{\boldsymbol{\gamma}}, \boldsymbol{\gamma}^{t}, \mathcal{S}^{t}, \boldsymbol{\theta}_t).
\end{align}

\bmhead{Lipschitz Continuous and Strongly Convex Loss Function $\mathcal{L}$}
Since the loss function $\mathcal{L}(\boldsymbol{\theta}; \tilde{\boldsymbol{x}})$ is $\mu$-strongly convex in $\boldsymbol{\theta}$ for every $\tilde{\boldsymbol{x}}$, from \Cref{strongly_convex_lemma} we know that $\max_{\tilde{\boldsymbol{x}}} \mathcal{L}(\boldsymbol{\theta}; \tilde{\boldsymbol{x}})$ is also going to be $\mu$-strongly convex.
Thus, using the additive property of strongly convex functions, we conclude that ${L(\boldsymbol{\theta}) = \sum_{i \in V} \max_{\tilde{\boldsymbol{x}}_i} \mathcal{L}(\boldsymbol{\theta}; \tilde{\boldsymbol{x}}_i)}$ is strongly convex.
It is straightforward to show that if $\max_{\tilde{\boldsymbol{x}}_i} \mathcal{L}(\boldsymbol{\theta}; \tilde{\boldsymbol{x}}_i)$ is $\mu$-strongly convex, then their summation $L(\boldsymbol{\theta})$ is strongly convex with parameter $\mu_n = n\mu$, where $n$ is the number of terms in $L(\boldsymbol{\theta})$.

Using the conclusions above, from the definition of strongly convex functions, we can write:
\begin{equation}\label{mu_strongly_convex_L}
    {\nabla_{\boldsymbol{\theta}}L(\boldsymbol{\theta}_t)}^{\intercal}\left(\boldsymbol{\theta}_t - \boldsymbol{\theta}^{*}\right) \geq L(\boldsymbol{\theta}_t) - L(\boldsymbol{\theta}^{*}) + \frac{\mu_n}{2}\norm{\boldsymbol{\theta}_t - \boldsymbol{\theta}^{*}}^2.
\end{equation}
Putting \Cref{eq:gd_base,mu_strongly_convex_L} together, we can deduce:
\begin{align}\nonumber\label{eq:gd_base_strongly_convex_L}
    L(\boldsymbol{\theta}_t) - L(\boldsymbol{\theta}^{*})
    &\leq \frac{\norm{\alpha_{t} \nabla_{\boldsymbol{\theta}}L_{\boldsymbol{\gamma}^{t}}(\boldsymbol{\theta}_{t})}^{2} + \norm{\boldsymbol{\theta}_{t} - \boldsymbol{\theta}^{*}}^{2} - \norm{\boldsymbol{\theta}_{t+1} - \boldsymbol{\theta}^{*}}^{2}}{2\alpha_{t}}\cdots\\\nonumber
    \cdots&+ \left(\nabla_{\boldsymbol{\theta}}L(\boldsymbol{\theta}_{t}) - \nabla_{\boldsymbol{\theta}}L_{\boldsymbol{\gamma}^{t}}(\boldsymbol{\theta}_{t}) \right)^{\intercal}\left(\boldsymbol{\theta}_{t} - \boldsymbol{\theta}^{*}\right) \cdots \\\cdots &- \frac{\mu_n}{2}\norm{\boldsymbol{\theta}_t - \boldsymbol{\theta}^{*}}^2.
\end{align}
Now, let us set the learning rate $\alpha_t = \frac{2}{\mu_n (t+1)}$.
We have:
\begin{align}\nonumber\label{eq:gd_base_strongly_convex_L_lr}
    L(\boldsymbol{\theta}_t) - L(\boldsymbol{\theta}^{*})
    &\leq \frac{1}{\mu_n (t+1)}\norm{ \nabla_{\boldsymbol{\theta}}L_{\boldsymbol{\gamma}^{t}}(\boldsymbol{\theta}_{t})}^{2}\cdots\\\nonumber
    \cdots&+ \frac{\mu_n (t+1)}{4}\norm{\boldsymbol{\theta}_{t} - \boldsymbol{\theta}^{*}}^{2} \cdots\\\nonumber
    \cdots&- \frac{\mu_n (t+1)}{4}\norm{\boldsymbol{\theta}_{t+1} - \boldsymbol{\theta}^{*}}^{2} \cdots\\\nonumber
    \cdots&+ \left(\nabla_{\boldsymbol{\theta}}L(\boldsymbol{\theta}_{t}) - \nabla_{\boldsymbol{\theta}}L_{\boldsymbol{\gamma}^{t}}(\boldsymbol{\theta}_{t}) \right)^{\intercal}\left(\boldsymbol{\theta}_{t} - \boldsymbol{\theta}^{*}\right) \cdots\\
    \cdots&- \frac{\mu_n}{2}\norm{\boldsymbol{\theta}_t - \boldsymbol{\theta}^{*}}^2.
\end{align}
Rearranging \Cref{eq:gd_base_strongly_convex_L_lr}, we can multiply both sides by a factor $t \geq 0$ to get:
\begin{align}\nonumber\label{eq:gd_base_strongly_convex_L_lr_t}
    t\left(L(\boldsymbol{\theta}_t) - L(\boldsymbol{\theta}^{*})\right)
    &\leq \frac{t}{\mu_n (t+1)}\norm{ \nabla_{\boldsymbol{\theta}}L_{\boldsymbol{\gamma}^{t}}(\boldsymbol{\theta}_{t})}^{2}\cdots\\\nonumber
    \cdots&+ \frac{\mu_n t(t-1)}{4}\norm{\boldsymbol{\theta}_{t} - \boldsymbol{\theta}^{*}}^{2}\cdots\\\nonumber
    \cdots&- \frac{\mu_n t(t+1)}{4}\norm{\boldsymbol{\theta}_{t+1} - \boldsymbol{\theta}^{*}}^{2}\\
    \cdots&+ t\left(\nabla_{\boldsymbol{\theta}}L(\boldsymbol{\theta}_{t}) - \nabla_{\boldsymbol{\theta}}L_{\boldsymbol{\gamma}^{t}}(\boldsymbol{\theta}_{t}) \right)^{\intercal}\left(\boldsymbol{\theta}_{t} - \boldsymbol{\theta}^{*}\right).
\end{align}
Since we assume that $\mathcal{L}(\boldsymbol{\theta}; \tilde{\boldsymbol{x}})$ is Lipschitz continuous with constant $\sigma$, from \Cref{lipschitz_lemma} we know that $\max_{\tilde{\boldsymbol{x}}} \mathcal{L}(\boldsymbol{\theta}; \tilde{\boldsymbol{x}})$ is also Lipschitz continuous with parameter $\sigma$.
Thus, similar to \Cref{eq:coreset_gradient_lipschitz} we can conclude that:
\begin{equation}\label{eq:coreset_gradient_lipschitz_rep}
        \norm{\nabla_{\boldsymbol{\theta}} L_{\boldsymbol{\gamma}^{t}}{(\boldsymbol{\theta}_{t})}} \leq \sum_{j \in \mathcal{S}^{t}} \gamma^{t}_{j} \sigma=\sigma,
\end{equation}
where once again we assume that the gradients are normalized so that $\sum_{j \in \mathcal{S}^{t}} \gamma^{t}_{j} = 1$.
Replacing \Cref{eq:coreset_gradient_lipschitz_rep} into \Cref{eq:gd_base_strongly_convex_L_lr_t}, we have:
\begin{align}\nonumber\label{eq:gd_base_strongly_convex_L_lr_t_lipschitz}
    t\left(L(\boldsymbol{\theta}_t) - L(\boldsymbol{\theta}^{*})\right)
    &\leq \frac{t\sigma^{2}}{\mu_n (t+1)} + \frac{\mu_n t(t-1)}{4}\norm{\boldsymbol{\theta}_{t} - \boldsymbol{\theta}^{*}}^{2} \cdots \\\nonumber  \cdots&- \frac{\mu_n t(t+1)}{4}\norm{\boldsymbol{\theta}_{t+1} - \boldsymbol{\theta}^{*}}^{2}\cdots\\
    \cdots&+ t\left(\nabla_{\boldsymbol{\theta}}L(\boldsymbol{\theta}_{t}) - \nabla_{\boldsymbol{\theta}}L_{\boldsymbol{\gamma}^{t}}(\boldsymbol{\theta}_{t}) \right)^{\intercal}\left(\boldsymbol{\theta}_{t} - \boldsymbol{\theta}^{*}\right).
\end{align}
Using the Cauchy-Schwarz inequality, we also write:
\begin{align}\nonumber\label{eq:gd_base_strongly_convex_L_lr_t_lipschitz_cauchy}
    t\left(L(\boldsymbol{\theta}_t) - L(\boldsymbol{\theta}^{*})\right)
    &\leq \frac{t\sigma^{2}}{\mu_n (t+1)}
    + \frac{\mu_n t(t-1)}{4}\norm{\boldsymbol{\theta}_{t} - \boldsymbol{\theta}^{*}}^{2}  \cdots \\\nonumber  \cdots&- \frac{\mu_n t(t+1)}{4}\norm{\boldsymbol{\theta}_{t+1} - \boldsymbol{\theta}^{*}}^{2}\cdots\\
    \cdots&+ t\norm{\nabla_{\boldsymbol{\theta}}L(\boldsymbol{\theta}_{t}) - \nabla_{\boldsymbol{\theta}}L_{\boldsymbol{\gamma}^{t}}(\boldsymbol{\theta}_{t}) }\norm{\boldsymbol{\theta}_{t} - \boldsymbol{\theta}^{*}}.
\end{align}
Adding all terms for $t = 0, 1, \dots, T-1$, we get:
\begin{align}\nonumber\label{eq:strongly_convex_bound_sum}
    \sum_{t=0}^{T-1} t\left(L(\boldsymbol{\theta}_t) - L(\boldsymbol{\theta}^{*})\right)
    & \leq \sum_{t=0}^{T-1}\frac{t\sigma^{2}}{\mu_n (t+1)}\cdots\\\nonumber
    \cdots&+ \sum_{t=0}^{T-1}\frac{\mu_n t(t-1)}{4}\norm{\boldsymbol{\theta}_{t} - \boldsymbol{\theta}^{*}}^{2}\\\nonumber \cdots &- \sum_{t=0}^{T-1}\frac{\mu_n t(t+1)}{4}\norm{\boldsymbol{\theta}_{t+1} - \boldsymbol{\theta}^{*}}^{2}\\
    \cdots&+ \sum_{t=0}^{T-1} t\norm{\nabla_{\boldsymbol{\theta}}L(\boldsymbol{\theta}_{t}) - \nabla_{\boldsymbol{\theta}}L_{\boldsymbol{\gamma}^{t}}(\boldsymbol{\theta}_{t}) }\norm{\boldsymbol{\theta}_{t} - \boldsymbol{\theta}^{*}}.
\end{align}
For the first sum we write:
\begin{align}\nonumber\label{eq:sum_1}
    \sum_{t=0}^{T-1}\frac{t\sigma^{2}}{\mu_n (t+1)} &= \frac{\sigma^{2}}{\mu_n}\sum_{t=0}^{T-1}\frac{t}{(t+1)}\\\nonumber
    &= \frac{\sigma^{2}}{\mu_n}\sum_{t=0}^{T-1}\left(1-\frac{1}{t+1}\right)\\\nonumber
    &\leq \frac{\sigma^{2}}{\mu_n}\sum_{t=0}^{T-1} 1\\
    &= \frac{\sigma^{2}T}{\mu_n}
\end{align}
The second and third terms on the RHS form a telescoping sum.
That is, defining ${a_{t} = \frac{\mu_n t(t-1)}{4}\norm{\boldsymbol{\theta}_{t} - \boldsymbol{\theta}^{*}}^{2}}$ we have:
\begin{align}\nonumber\label{eq:sum_2}
    \sum_{t=0}^{T-1}\frac{\mu_n t(t-1)}{4}\norm{\boldsymbol{\theta}_{t} - \boldsymbol{\theta}^{*}}^{2}- \sum_{t=0}^{T-1}\frac{\mu_n t(t+1)}{4}\norm{\boldsymbol{\theta}_{t+1} - \boldsymbol{\theta}^{*}}^{2}
    &= \sum_{t=0}^{T-1} \left(a_{t} - a_{t+1}\right) \\\nonumber
    &= a_{0} - a_{T}\\
    &= -\frac{\mu_n T(T-1)}{4}\norm{\boldsymbol{\theta}_{T} - \boldsymbol{\theta}^{*}}^{2}.
\end{align}
Combining \Cref{eq:strongly_convex_bound_sum,eq:sum_1,eq:sum_2}, we deduce:
\begin{align}\nonumber\label{eq:strongly_convex_bound_sum_post}
    \sum_{t=0}^{T-1} t\left(L(\boldsymbol{\theta}_t) - L(\boldsymbol{\theta}^{*})\right) &\leq \frac{\sigma^{2}T}{\mu_n} -\frac{\mu_n T(T-1)}{4}\norm{\boldsymbol{\theta}_{T} - \boldsymbol{\theta}^{*}}^{2}\cdots\\
    \cdots&+ \sum_{t=0}^{T-1} t\norm{\nabla_{\boldsymbol{\theta}}L(\boldsymbol{\theta}_{t}) - \nabla_{\boldsymbol{\theta}}L_{\boldsymbol{\gamma}^{t}}(\boldsymbol{\theta}_{t}) }\norm{\boldsymbol{\theta}_{t} - \boldsymbol{\theta}^{*}}.
\end{align}
Using our assumption $0\leq\norm{\boldsymbol{\theta} - \boldsymbol{\theta}^{*}} \leq \Delta$, we then have:
\begin{align}\label{eq:strongly_convex_bound_prelim}
    \sum_{t=0}^{T-1} t\left(L(\boldsymbol{\theta}_t) - L(\boldsymbol{\theta}^{*})\right) \leq \frac{\sigma^{2}T}{\mu_n} + \sum_{t=0}^{T-1} t\Delta\norm{\nabla_{\boldsymbol{\theta}}L(\boldsymbol{\theta}_{t}) - \nabla_{\boldsymbol{\theta}}L_{\boldsymbol{\gamma}^{t}}(\boldsymbol{\theta}_{t})}.
\end{align}
To simplify the LHS of \Cref{eq:strongly_convex_bound_prelim}, note that $\min_{t=0: T-1} L(\boldsymbol{\theta}_{t})-L(\boldsymbol{\theta}^{*}) \leq L(\boldsymbol{\theta}_t) - L(\boldsymbol{\theta}^{*})$.
Thus, we can have:
\begin{align}\nonumber\label{eq:strongly_convex_bound_LHS}
    \sum_{t=0}^{T-1} t&\left(L(\boldsymbol{\theta}_t) - L(\boldsymbol{\theta}^{*})\right) \\\nonumber
    &\geq \sum_{t=0}^{T-1} t\left(\min_{t=0: T-1} L(\boldsymbol{\theta}_{t})-L(\boldsymbol{\theta}^{*})\right)\\\nonumber
    &\geq\left(\min_{t=0: T-1} L(\boldsymbol{\theta}_{t})-L(\boldsymbol{\theta}^{*})\right)\sum_{t=0}^{T-1} t\\
    &\geq\left(\min_{t=0: T-1} L(\boldsymbol{\theta}_{t})-L(\boldsymbol{\theta}^{*})\right)\left(\frac{T(T-1)}{2}\right).
\end{align}
After putting \Cref{eq:strongly_convex_bound_LHS} into \Cref{eq:strongly_convex_bound_prelim} and rearranging the terms, we will finally arrive at:
\begin{align}\label{eq:strongly_convex_bound_final}
    \min_{t=0: T-1} L(\boldsymbol{\theta}_{t})-L(\boldsymbol{\theta}^{*}) \leq \frac{2\sigma^{2}}{\mu_n (T-1)} + \sum_{t=0}^{T-1} \frac{2\Delta t}{T(T-1)}\Gamma(L, L_{\boldsymbol{\gamma}}, \boldsymbol{\gamma}^{t}, \mathcal{S}^{t}, \boldsymbol{\theta}_t).
\end{align}
\end{proof}

\section{Implementation Details}\label{ap:sec:imp_det}

This section provides the details of our experiments in \Cref{sec:experiments}.
We used a single NVIDIA Tesla V100-SXM2-16GB GPU for CIFAR-10~\citep{krizhevsky2009learning} and SVHN~\citep{netzer2011reading}, a single NVIDIA Tesla V100-SXM2-32GB GPU for ImageNet-12~\citep{liu2020refool}, and a single NVIDIA A100 for ImageNet-100~\citep{russakovsky2015imagenet}.
Our implementation is released at this \href{repository}{https://github.com/hmdolatabadi/ACS}.\footnote{\textcolor{revision}{Note that the exact training time reported in our experiments may vary between different devices, their drivers, the version of libraries used, etc. Nevertheless, the relative time difference and speed-up between training with the entire data and its subsets using adversarial coreset selection is going to be similar.}}

\subsection{Training Settings}
\textcolor{revision}{\Cref{tab:hyperI,tab:hyperII} show the entire set of hyper-parameters and settings used for training the models of \Cref{sec:experiments}.}

\subsection{Evaluation Settings}
For the evaluation of TRADES and $\ell_p$-PGD adversarial training, we use PGD attacks.
In particular, for TRADES and $\ell_\infty$-PGD adversarial training, we use $\ell_\infty$-PGD attacks with $\varepsilon=8/255$, step-size $\alpha=1/255$, 50 iterations, and 10 random restarts.
\textcolor{revision}{The only exception is ImageNet-100, where we evaluate the models using $\ell_\infty$-PGD with $\varepsilon=4/255$ and step-size $\alpha=2/255$.}
Also, for $\ell_2$-PGD adversarial training we use $\ell_2$-PGD attacks with $\varepsilon=80/255$, step-size $\alpha=8/255$, 50 iterations and 10 random restarts.

For Perceptual Adversarial Training~(PAT), we report the attacks' settings in \Cref{tab:hyper_PAT}.
Note that we chose the same set of unseen/seen attacks to evaluate each case.
However, since we trained our model with slightly different $\varepsilon$ bounds, we changed the attacks' settings accordingly.

\begin{table*}[tb!]
	\caption{Hyper-parameters of attacks used for the evaluation of PAT models in \Cref{sec:experiments}.}
	\label{tab:hyper_PAT}
	\begin{center}
		\begin{small}
			\setlength\tabcolsep{1.6pt}
			\begin{tabular}{cccc}
				\toprule
				\textbf{\scriptsize{Dataset}}
				& \textbf{Attack}
				& \textbf{Bound}
				& \textbf{Iterations} \\
				\midrule
				\parbox[t]{2mm}{\multirow{6}{*}{\rotatebox[origin=c]{90}{CIFAR-10}}}
				& AutoAttack-$\ell_2$~\citep{croce2020autoattack}                   & $1$	           & $20$ \\
				& AutoAttack-$\ell_\infty$~\citep{croce2020autoattack}              & $8/255$	       & $20$ \\
				& StAdv~\citep{xiao2018stadv}                                       & $0.02$	       & $50$ \\
				& ReColor~\citep{laidlaw2019recolor}                                & $0.06$	       & $100$ \\
				& PPGD~\citep{laidlaw2021pat}                                       & $0.40$	       & $40$ \\
				& LPA~\citep{laidlaw2021pat}                                        & $0.40$	       & $40$ \\                     
				\midrule
				\parbox[t]{2mm}{\multirow{7}{*}{\rotatebox[origin=c]{90}{ImageNet-12}}}
				& AutoAttack-$\ell_2$~\citep{croce2020autoattack}                   & $1200/255$	   & $20$ \\
				& AutoAttack-$\ell_\infty$~\citep{croce2020autoattack}              & $4/255$	       & $20$ \\
				& JPEG~\citep{kang2019JPEG}                                         & $0.125$	       & $200$ \\
				& StAdv~\citep{xiao2018stadv}                                       & $0.02$	       & $50$ \\
				& ReColor~\citep{laidlaw2019recolor}                                & $0.06$	       & $200$ \\
				& PPGD~\citep{laidlaw2021pat}                                       & $0.35$	       & $40$ \\
				& LPA~\citep{laidlaw2021pat}                                        & $0.35$	       & $40$ \\
				\bottomrule
			\end{tabular}
		\end{small}
	\end{center}
\end{table*}

\begin{sidewaystable*}[p!]
	\caption{Training details for experimental results of \Cref{sec:experiments}.}
	\label{tab:hyperI}
	\begin{center}
		\begin{scriptsize}
			\begin{tabular}{lcccccc}
				\toprule
				\multirow{2}{*}{\textbf{Hyperparameter}}    & \multicolumn{6}{c}{\textbf{Experiment}}\\
				\cmidrule{2-7}
				& TRADES                                    & $\ell_\infty$-PGD                & $\ell_2$-PGD                    & Fast-LPA                   & Fast-LPA                  & Fast Adv.\\
				\midrule
				\textbf{Dataset}                            & CIFAR-10                         & CIFAR-10                        & SVHN                       & CIFAR-10                  & ImageNet-12 & CIFAR-10\\
				\textbf{Model Arch.}                        & ResNet-18                        & ResNet-18                       & ResNet-18                  & ResNet-50                 & ResNet50    & ResNet-18\\
				\midrule
				\textbf{Optimizer}                          & SGD                              & SGD                             & SGD                        & SGD                       & SGD         & SGD\\
				\textbf{Scheduler}                          & Multi-step                       & Multi-step                      & Multi-step                 & Multi-step                & Multi-step  & Multi-step\\
				\textbf{Initial lr.}                        & $0.1$                            & $0.01$                          & $0.1$                      & $0.1$                     & $0.1$       & $0.1$\\
				\textbf{lr. Decay (epochs)}                 & $0.1$ ($75$, $90$)               & $0.1$ ($80$, $100$)             & $0.1$ ($75$, $90$, $100$)  & $0.1$ ($75$, $90$, $100$) & $0.1$ ($45$, $60$, $80$)) & $0.1$ ($37$, $56$)\\
				\textbf{Weight Decay}                       & $2\cdot10^{-4}$                  & $5\cdot10^{-4}$                 & $5\cdot10^{-4}$            & $2\cdot10^{-4}$           & $2\cdot10^{-4}$ & $5\cdot10^{-4}$\\
				\textbf{Batch Size (full)}                  & $128$                            & $128$                           & $128$                      & $50$                      & $50$      & $128$\\
				\textbf{Total Epochs}                       & $100$                            & $120$                           & $120$                      & $120$                     & $90$      & $60$\\
				\midrule
				\textbf{Coreset Size}                       & $50$\%                           & $50$\%                          & $30$\%                     & $40$\%                     & $50$\%   & $50$\%\\
				\textbf{Coreset Batch Size}                 & $20$                             & $20$                            & $20$                       & $20$                       & $20$     & $20$\\
				\textbf{Warm-start Epochs}                  & $30$                             & $36$                            & $22$                       & $29$                       & $27$     & $22$\\
				\textbf{Coreset Selection Freq. (epochs)}  & $20$                              & $20$                            & $20$                       & $10$                       & $15$     & $5$\\
				\midrule
				\textbf{Visual Similarity Measure}            & $\ell_\infty$                  & $\ell_\infty$                   & $\ell_2$                   & LPIPS (AlexNet) & LPIPS (AlexNet) & $\ell_\infty$\\
				\textbf{$\varepsilon$ (Bound on Visual Sim.)} & $8/255$                        & $8/255$                         & $80/255$                   & $0.5$                      & $0.25$   & $8/255$\\
				\textbf{Attack Iterations (Training)}         & $10$                           & $10$                            & $10$                       & $10$                       & $10$     & $1$\\
				\textbf{Attack Iterations (Coreset Selection)}& $10$                           & $1$                             & $10$                       & $10$                       & $10$     & $1$\\
				\textbf{Attack Step-size}                     & $1.785/255$                     & $1.25/255$                      & $8/255$                   & -                          & -        & $10/255$\\      
				\bottomrule
			\end{tabular}
		\end{scriptsize}
	\end{center}
\end{sidewaystable*}

\begin{sidewaystable*}[p!]
	\caption{\textcolor{revision}{Training details for experimental results of \Cref{sec:experiments}.}}
	\label{tab:hyperII}
	\begin{center}
		\begin{scriptsize}
			\begin{tabular}{lcccccc}
				\toprule
				\multirow{2}{*}{\textbf{Hyperparameter}}    & \multicolumn{4}{c}{\textbf{Experiment}}\\
				\cmidrule{2-6}
				                                            & TRADES                          & $\ell_\infty$-PGD               & $\ell_2$-PGD                & $\ell_\infty$-PGD  & TRADES\\
				\midrule
				\textbf{Dataset}                            & CIFAR-10                         & CIFAR-10                        & SVHN                       & ImageNet-100       & ImageNet-100\\
				\textbf{Model Arch.}                        & WideResNet-28-10                 & WideResNet-28-10                & WideResNet-28-10           & ResNet-34          & ResNet-50\\
				\midrule
				\textbf{Optimizer}                          & SGD                              & SGD                             & SGD                        & SGD                & SGD\\
				\textbf{Scheduler}                          & Multi-step                       & Multi-step                      & Multi-step                 & Multi-step         & Multi-step\\
				\textbf{Initial lr.}                        & $0.1$                            & $0.1$                           & $0.1$                      & $0.1$              & $0.1$\\
				\textbf{lr. Decay (epochs)}                 & $0.1$ ($75$, $90$)               & $0.1$ ($80$, $100$)             & $0.1$ ($75$, $90$, $100$)  & $0.1$ ($20$, $40$) & $0.1$ ($30$, $60$)\\
				\textbf{Weight Decay}                       & $2\cdot10^{-4}$                  & $5\cdot10^{-4}$                 & $5\cdot10^{-4}$            & $2\cdot10^{-4}$    & $2\cdot10^{-4}$\\
				\textbf{Batch Size (full)}                  & $128$                            & $128$                           & $128$                      & $128$              & $256$\\
				\textbf{Total Epochs}                       & $100$                            & $120$                           & $120$                      & $50$               & $90$\\
				\midrule
				\textbf{Coreset Size}                       & $50$\%                           & $50$\%                          & $30$\%                     & $56$\%             & $50$\%\\
				\textbf{Coreset Batch Size}                 & $128$                            & $128$                           & $128$                      & $128$              & $256$\\
				\textbf{Warm-start Epochs}                  & $30$                             & $36$                            & $22$                       & $25$               & $27$\\
				\textbf{Coreset Selection Freq. (epochs)}   & $20$                             & $20$                            & $20$                       & $3$                & $15$ \\
				\midrule
				\textbf{Visual Similarity Measure}            & $\ell_\infty$                  & $\ell_\infty$                   & $\ell_2$                   & $\ell_\infty$       & $\ell_\infty$\\
				\textbf{$\varepsilon$ (Bound on Visual Sim.)} & $8/255$                        & $8/255$                         & $80/255$                   & $4/255$             & $4/255$\\
				\textbf{Attack Iterations (Training)}         & $10$                           & $10$                            & $10$                       & $2$                 & $2$\\
				\textbf{Attack Iterations (Coreset Selection)}& $10$                           & $1$                             & $10$                       & $1$                 & $1$\\
				\textbf{Attack Step-size}                     & $1.785/255$                    & $1.25/255$                      & $8/255$                    & $2/255$             & $2/255$\\      
				\bottomrule
			\end{tabular}
		\end{scriptsize}
	\end{center}
\end{sidewaystable*}

\section{Extended Experimental Results}\label{ap:extended}
\Cref{tab:LPA} shows the full details of our experiments on PAT~\citep{laidlaw2021pat}.
In each case, we train ResNet-50~\citep{he2016deep} classifiers using LPIPS~\citep{zhang2018lpips} objective of PAT~\citep{laidlaw2021pat}.
In each dataset, all the training hyper-parameters are fixed.
The only difference is that we enable adversarial coreset selection for our method.
During inference, we evaluate each trained model against a few unseen attacks, as well as two variants of Perceptual Adversarial Attacks~\citep{laidlaw2021pat} that the models are trained initially on.
As can be seen, adversarial coreset selection can significantly reduce the training time while experiencing only a small reduction in the average robust accuracy.

\begin{sidewaystable*}[p!]
	\caption{Clean and robust accuracy (\%), and total training time (mins) of Perceptual Adversarial Training for CIFAR-10 and ImageNet-12 datasets. The training objective uses Fast Lagrangian Perceptual Attack (LPA)~\citep{laidlaw2021pat} to train the network. At test time, the networks are evaluated against attacks not seen during training and different versions of Perceptual Adversarial Attack (PPGD and LPA). The unseen attacks were selected similar to \citet{laidlaw2021pat} in each case. Please see the \Cref{ap:sec:imp_det} for more information about the settings.}
	\label{tab:LPA}
	\begin{center}
		\begin{scriptsize}
			\setlength\tabcolsep{5pt}
			\def\arraystretch{2.5}
			\begin{tabular}{cccccccccccc}
				\toprule
				\parbox[t]{2mm}{\multirow{2}{*}{\rotatebox[origin=c]{90}{\textbf{\scriptsize{Dataset}}}}}
				&\multirow{2}{*}{\textbf{Training}}
				& \multirow{2}{*}{\textbf{Clean}}
				& \multicolumn{6}{c}{\textbf{Unseen Attacks}}
				& \multicolumn{2}{c}{\textbf{Seen Attacks}} 
				& \multirow{2}{*}{\shortstack{\textbf{Train. Time}\\(mins)}}\\
				\cmidrule{4-9} \cmidrule{10-11}
				& & & Auto-$\ell_2$   & Auto-$\ell_\infty$      & JPEG    & StAdv
				& ReColor   & Mean    & PPGD     & LPA      & \\
				\midrule
				\parbox[t]{2mm}{\multirow{3}{*}{\rotatebox[origin=c]{90}{\scriptsize{\textbf{CIFAR-10}}}}}
				& Adv. \textsc{Craig} (Ours)             & $\mathbf{83.21}$        & $39.98$	& $33.94$	       &  -           & $49.60$	 & $62.69$	  & $\mathbf{46.55}$  & $19.56$	& $7.42$       & $\mathbf{767.34}$ \\
				& Adv. \textsc{GradMatch} (Ours)         & $\mathbf{83.14}$	       & $39.20$	& $34.11$	       &  -           & $48.86$	 & $62.26$	  & $\mathbf{46.11}$   & $19.94$	& $7.54$   & $\mathbf{787.26}$ \\
				& Full PAT (Fast-LPA)                    & $\mathbf{86.02}$	       & $43.27$	& $37.96$	       &  -           & $48.68$	 & $62.23$	  & $\mathbf{48.04}$   & $22.62$	& $8.01$   & $\mathbf{1682.94}$ \\
				\midrule
				\parbox[t]{2mm}{\multirow{3}{*}{\rotatebox[origin=c]{90}{\scriptsize{\textbf{ImageNet-12}}}}}
				& Adv. \textsc{Craig} (Ours)             & $\mathbf{86.99}$        & $51.54$	& $60.42$	       & $71.79$      & $37.47$	 & $44.04$	  & $\mathbf{53.05}$    & $29.04$    & $14.07$   & $\mathbf{2817.06}$ \\
				& Adv. \textsc{GradMatch} (Ours)         & $\mathbf{87.08}$	       & $51.38$	& $60.64$          & $72.15$      & $35.83$	 & $45.83$	  & $\mathbf{53.17}$    & $28.36$    & $13.11$   & $\mathbf{2865.72}$ \\
				& Full PAT (Fast-LPA)                    & $\mathbf{91.22}$	       & $57.37$    & $66.89$          & $76.25$      & $19.29$	 & $46.35$	  & $\mathbf{53.23}$    & $33.17$    & $13.49$   & $\mathbf{5613.12}$ \\
				\bottomrule
			\end{tabular}
		\end{scriptsize}
	\end{center}
\end{sidewaystable*}

\end{appendices}

\clearpage
\newpage
\bmhead{Data Availability Statement}
All the datasets used in this paper are publicly available.

\bmhead{Code Availability Statement}
Our implementation can be found in this repository: \href{tinyurl.com/33edneuf}{https://github.com/hmdolatabadi/ACS}.

\begin{small}
\bibliographystyle{sn-basic}
\bibliography{references}

\begin{thebibliography}{49}
\providecommand{\natexlab}[1]{#1}
\providecommand{\url}[1]{{#1}}
\providecommand{\urlprefix}{URL }
\providecommand{\doi}[1]{\url{https://doi.org/#1}}
\providecommand{\eprint}[2][]{\url{#2}}
 \bibcommenthead

\bibitem[{Adadi(2021)}]{adadi2021data}
Adadi A (2021) A survey on data-efficient algorithms in big data era. Journal
  of Big Data 8(1):1--54

\bibitem[{Andriushchenko and
  Flammarion(2020)}]{andriushchenko2020understanding}
Andriushchenko M, Flammarion N (2020) Understanding and improving fast
  adversarial training. In: Proceedings of the Advances in Neural Information
  Processing Systems 33: Annual Conference on Neural Information Processing
  Systems

\bibitem[{Biggio et~al.(2013)Biggio, Corona, Maiorca, Nelson, Srndic, Laskov,
  Giacinto, and Roli}]{biggio2013evasion}
Biggio B, Corona I, Maiorca D, et~al. (2013) Evasion attacks against machine
  learning at test time. In: Proceedings of the European Conference on Machine
  Learning and Knowledge Discovery in Databases ({ECML}-{PKDD}), pp 387--402

\bibitem[{Campbell and Broderick(2018)}]{campbell2018coreset}
Campbell T, Broderick T (2018) Bayesian coreset construction via greedy
  iterative geodesic ascent. In: Proceedings of the 35th International
  Conference on Machine Learning ({ICML}), pp 697--705

\bibitem[{Croce and Hein(2020)}]{croce2020autoattack}
Croce F, Hein M (2020) Reliable evaluation of adversarial robustness with an
  ensemble of diverse parameter-free attacks. In: Proceedings of the 37th
  International Conference on Machine Learning ({ICML}) 2020, pp 2206--2216

\bibitem[{Danskin(1967)}]{danskin1967theory}
Danskin JM (1967) The theory of max-min and its application to weapons
  allocation problems, vol~5. Springer Science \& Business Media

\bibitem[{Elenberg et~al.(2016)Elenberg, Khanna, Dimakis, and
  Negahban}]{elenberg2016restricted}
Elenberg ER, Khanna R, Dimakis AG, et~al. (2016) Restricted strong convexity
  implies weak submodularity. CoRR abs/1612.00804

\bibitem[{Eykholt et~al.(2018)Eykholt, Evtimov, Fernandes, Li, Rahmati, Xiao,
  Prakash, Kohno, and Song}]{eykholt2018robust}
Eykholt K, Evtimov I, Fernandes E, et~al. (2018) Robust physical-world attacks
  on deep learning visual classification. In: Proceeding of the {IEEE}
  Conference on Computer Vision and Pattern Recognition ({CVPR}), pp 1625--1634

\bibitem[{Feldman(2020)}]{feldman2011coresets}
Feldman D (2020) Introduction to core-sets: an updated survey. CoRR
  abs/2011.09384

\bibitem[{Goodfellow et~al.(2015)Goodfellow, Shlens, and
  Szegedy}]{goodfellow2014explaining}
Goodfellow IJ, Shlens J, Szegedy C (2015) Explaining and harnessing adversarial
  examples. In: Proceedings of the 3rd International Conference on Learning
  Representations ({ICLR})

\bibitem[{Har{-}Peled and Mazumdar(2004)}]{harpeled2004oncoresets}
Har{-}Peled S, Mazumdar S (2004) On coresets for k-means and k-median
  clustering. In: Proceedings of the 36th Annual {ACM} Symposium on Theory of
  Computing ({STOC}), pp 291--300

\bibitem[{He et~al.(2016)He, Zhang, Ren, and Sun}]{he2016deep}
He K, Zhang X, Ren S, et~al. (2016) Deep residual learning for image
  recognition. In: Proceedings of the {IEEE} Conference on Computer Vision and
  Pattern Recognition ({CVPR}), pp 770--778

\bibitem[{de~Jorge~Aranda et~al.(2022)de~Jorge~Aranda, Bibi, Volpi, Sanyal,
  Torr, Rogez, and Dokania}]{aranda2022nfgsm}
de~Jorge~Aranda P, Bibi A, Volpi R, et~al. (2022) Make some noise: Reliable and
  efficient single-step adversarial training. In: Proceedings of the Advances
  in Neural Information Processing Systems 35: Annual Conference on Neural
  Information Processing Systems ({NeurIPS})

\bibitem[{Kang et~al.(2019)Kang, Sun, Hendrycks, Brown, and
  Steinhardt}]{kang2019JPEG}
Kang D, Sun Y, Hendrycks D, et~al. (2019) Testing robustness against unforeseen
  adversaries. CoRR abs/1908.08016

\bibitem[{Karras et~al.(2020)Karras, Laine, Aittala, Hellsten, Lehtinen, and
  Aila}]{karras2020stylegan2}
Karras T, Laine S, Aittala M, et~al. (2020) Analyzing and improving the image
  quality of stylegan. In: Proceedings of the {IEEE} Conference on Computer
  Vision and Pattern Recognition ({CVPR}), pp 8107--8116

\bibitem[{Katharopoulos and Fleuret(2018)}]{katharopoulos2018notall}
Katharopoulos A, Fleuret F (2018) Not all samples are created equal: Deep
  learning with importance sampling. In: Proceedings of the 35th International
  Conference on Machine Learning ({ICML}), pp 2530--2539

\bibitem[{Killamsetty et~al.(2021{\natexlab{a}})Killamsetty, Sivasubramanian,
  Ramakrishnan, De, and Iyer}]{killamsetty2021gradmatch}
Killamsetty K, Sivasubramanian D, Ramakrishnan G, et~al. (2021{\natexlab{a}})
  {GRAD-MATCH:} gradient matching based data subset selection for efficient
  deep model training. In: Proceedings of the 38th International Conference on
  Machine Learning ({ICML}), pp 5464--5474

\bibitem[{Killamsetty et~al.(2021{\natexlab{b}})Killamsetty, Sivasubramanian,
  Ramakrishnan, and Iyer}]{killamsetty2021glister}
Killamsetty K, Sivasubramanian D, Ramakrishnan G, et~al. (2021{\natexlab{b}})
  {GLISTER:} generalization based data subset selection for efficient and
  robust learning. In: Proceedings of the 35th {AAAI} Conference on Artificial
  Intelligence, pp 8110--8118

\bibitem[{Killamsetty et~al.(2021{\natexlab{c}})Killamsetty, Zhao, Chen, and
  Iyer}]{killamsetty2021retrieve}
Killamsetty K, Zhao X, Chen F, et~al. (2021{\natexlab{c}}) {RETRIEVE:} coreset
  selection for efficient and robust semi-supervised learning. In: Advances in
  Neural Information Processing Systems 34: Annual Conference on Neural
  Information Processing Systems ({NeurIPS}), pp 14,488--14,501

\bibitem[{Kolter and Madry(2018)}]{madry2018adversarial}
Kolter Z, Madry A (2018) Adversarial robustness: Theory and practice.
  \url{https://adversarial-ml-tutorial.org/}, tutorial in the Advances in
  Neural Information Processing Systems 31: Annual Conference on Neural
  Information Processing Systems ({NeurIPS})

\bibitem[{Krizhevsky and Hinton(2009)}]{krizhevsky2009learning}
Krizhevsky A, Hinton G (2009) Learning multiple layers of features from tiny
  images. Master's thesis, Department of Computer Science, University of
  Toronto

\bibitem[{Laidlaw and Feizi(2019)}]{laidlaw2019recolor}
Laidlaw C, Feizi S (2019) Functional adversarial attacks. In: Proceedings of
  the Advances in Neural Information Processing Systems 32: Annual Conference
  on Neural Information Processing Systems ({NeurIPS}), pp 10,408--10,418

\bibitem[{Laidlaw et~al.(2021)Laidlaw, Singla, and Feizi}]{laidlaw2021pat}
Laidlaw C, Singla S, Feizi S (2021) Perceptual adversarial robustness: Defense
  against unseen threat models. In: Proceedings of the 9th International
  Conference on Learning Representations ({ICLR})

\bibitem[{Liu et~al.(2020)Liu, Ma, Bailey, and Lu}]{liu2020refool}
Liu Y, Ma X, Bailey J, et~al. (2020) Reflection backdoor: {A} natural backdoor
  attack on deep neural networks. In: Proceedings of the 16th European
  Conference on Computer Vision ({ECCV}), pp 182--199

\bibitem[{Ma et~al.(2021)Ma, Niu, Gu, Wang, Zhao, Bailey, and
  Lu}]{ma2021understanding}
Ma X, Niu Y, Gu L, et~al. (2021) Understanding adversarial attacks on deep
  learning based medical image analysis systems. Pattern Recognition
  110:107,332

\bibitem[{Madry et~al.(2018)Madry, Makelov, Schmidt, Tsipras, and
  Vladu}]{madry2018towards}
Madry A, Makelov A, Schmidt L, et~al. (2018) Towards deep learning models
  resistant to adversarial attacks. In: Proceedings of the 6th International
  Conference on Learning Representations ({ICLR})

\bibitem[{Minoux(1978)}]{minoux1978accelerated}
Minoux M (1978) Accelerated greedy algorithms for maximizing submodular set
  functions. In: Optimization Techniques. Springer, p 234--243

\bibitem[{Mirzasoleiman et~al.(2020{\natexlab{a}})Mirzasoleiman, Bilmes, and
  Leskovec}]{mirzasoleiman2020craig}
Mirzasoleiman B, Bilmes JA, Leskovec J (2020{\natexlab{a}}) Coresets for
  data-efficient training of machine learning models. In: Proceedings of the
  37th International Conference on Machine Learning ({ICML}), pp 6950--6960

\bibitem[{Mirzasoleiman et~al.(2020{\natexlab{b}})Mirzasoleiman, Cao, and
  Leskovec}]{mirzasoleiman2020crust}
Mirzasoleiman B, Cao K, Leskovec J (2020{\natexlab{b}}) Coresets for robust
  training of deep neural networks against noisy labels. In: Proceedings of the
  Advances in Neural Information Processing Systems 33: Annual Conference on
  Neural Information Processing Systems ({NeurIPS})

\bibitem[{Nemhauser et~al.(1978)Nemhauser, Wolsey, and
  Fisher}]{nemhauser1978analysis}
Nemhauser GL, Wolsey LA, Fisher ML (1978) An analysis of approximations for
  maximizing submodular set functions - {I}. Mathematical Programming
  14(1):265--294

\bibitem[{Netzer et~al.(2011)Netzer, Wang, Coates, Bissacco, Wu, and
  Ng}]{netzer2011reading}
Netzer Y, Wang T, Coates A, et~al. (2011) Reading digits in natural images with
  unsupervised feature learning. In: NeurIPS Workshop on Deep Learning and
  Unsupervised Feature Learning

\bibitem[{Pati et~al.(1993)Pati, Rezaiifar, and Krishnaprasad}]{pati1992omp}
Pati YC, Rezaiifar R, Krishnaprasad PS (1993) Orthogonal matching pursuit:
  recursive function approximation with applications to wavelet decomposition.
  In: Proceedings of 27th Asilomar Conference on Signals, Systems and
  Computers, pp 40--44

\bibitem[{Qin et~al.(2019)Qin, Martens, Gowal, Krishnan, Dvijotham, Fawzi, De,
  Stanforth, and Kohli}]{qin2019adversarial}
Qin C, Martens J, Gowal S, et~al. (2019) Adversarial robustness through local
  linearization. In: Proceedings of the Advances in Neural Information
  Processing Systems 32: Annual Conference on Neural Information Processing
  Systems ({NeurIPS})

\bibitem[{Russakovsky et~al.(2015)Russakovsky, Deng, Su, Krause, Satheesh, Ma,
  Huang, Karpathy, Khosla, Bernstein, Berg, and Li}]{russakovsky2015imagenet}
Russakovsky O, Deng J, Su H, et~al. (2015) Image{N}et large scale visual
  recognition challenge. International Journal of Computer Vision
  115(3):211--252

\bibitem[{Schwartz et~al.(2020)Schwartz, Dodge, Smith, and
  Etzioni}]{schwartz2020greenai}
Schwartz R, Dodge J, Smith NA, et~al. (2020) Green {AI}. Communication of the
  {ACM} 63(12):54--63

\bibitem[{Smith(2017)}]{smith2017cyclic}
Smith LN (2017) Cyclical learning rates for training neural networks. In:
  Proceedings of the {IEEE} Winter Conference on Applications of Computer
  Vision ({WACV}), pp 464--472

\bibitem[{Strubell et~al.(2019)Strubell, Ganesh, and
  McCallum}]{strubell2019energy}
Strubell E, Ganesh A, McCallum A (2019) Energy and policy considerations for
  deep learning in {NLP}. In: Proceedings of the 57th Conference of the
  Association for Computational Linguistics {ACL}, pp 3645--3650

\bibitem[{Szegedy et~al.(2014)Szegedy, Zaremba, Sutskever, Bruna, Erhan,
  Goodfellow, and Fergus}]{szegedy2014intriguing}
Szegedy C, Zaremba W, Sutskever I, et~al. (2014) Intriguing properties of
  neural networks. In: Proceedings of the 2nd International Conference on
  Learning Representations ({ICLR})

\bibitem[{Tram{\`{e}}r et~al.(2018)Tram{\`{e}}r, Kurakin, Papernot, Goodfellow,
  Boneh, and McDaniel}]{tramer2018ensemble}
Tram{\`{e}}r F, Kurakin A, Papernot N, et~al. (2018) Ensemble adversarial
  training: Attacks and defenses. In: Proceedings of the 6th International
  Conference on Learning Representations ({ICLR})

\bibitem[{Tsipras et~al.(2019)Tsipras, Santurkar, Engstrom, Turner, and
  Madry}]{tsipras2019robustness}
Tsipras D, Santurkar S, Engstrom L, et~al. (2019) Robustness may be at odds
  with accuracy. In: Proceedings of the 7th International Conference on
  Learning Representations ({ICLR})

\bibitem[{Vahdat and Kautz(2020)}]{vahdat2020nvae}
Vahdat A, Kautz J (2020) {NVAE:} {A} deep hierarchical variational autoencoder.
  In: Proceedings of the Advances in Neural Information Processing Systems 33:
  Annual Conference on Neural Information Processing Systems ({NeurIPS})

\bibitem[{Wei et~al.(2015)Wei, Iyer, and Bilmes}]{wei2015nnc}
Wei K, Iyer R, Bilmes J (2015) Submodularity in data subset selection and
  active learning. In: Proceedings of the 32nd International Conference on
  Machine Learning ({ICML}), pp 1954--1963

\bibitem[{Wolsey(1982)}]{wolsey1982greedy}
Wolsey LA (1982) An analysis of the greedy algorithm for the submodular set
  covering problem. Combinatorica 2(4):385--393

\bibitem[{Wong et~al.(2020)Wong, Rice, and Kolter}]{wong2020fast}
Wong E, Rice L, Kolter JZ (2020) Fast is better than free: Revisiting
  adversarial training. In: Proceedings of the 8th International Conference on
  Learning Representations ({ICLR})

\bibitem[{Wu et~al.(2019)Wu, Kirillov, Massa, Lo, and
  Girshick}]{wu2019detectron2}
Wu Y, Kirillov A, Massa F, et~al. (2019) Detectron2.
  \url{https://github.com/facebookresearch/detectron2}

\bibitem[{Xiao et~al.(2018)Xiao, Zhu, Li, He, Liu, and Song}]{xiao2018stadv}
Xiao C, Zhu J, Li B, et~al. (2018) Spatially transformed adversarial examples.
  In: Proceedings of the 6th International Conference on Learning
  Representations ({ICLR})

\bibitem[{Zagoruyko and Komodakis(2016)}]{zagoruyko2016wresnet}
Zagoruyko S, Komodakis N (2016) Wide residual networks. In: Proceedings of the
  British Machine Vision Conference ({BMVC})

\bibitem[{Zhang et~al.(2019)Zhang, Yu, Jiao, Xing, Ghaoui, and
  Jordan}]{zhang2019trades}
Zhang H, Yu Y, Jiao J, et~al. (2019) Theoretically principled trade-off between
  robustness and accuracy. In: Proceedings of the 36th International Conference
  on Machine Learning ({ICML}), pp 7472--7482

\bibitem[{Zhang et~al.(2018)Zhang, Isola, Efros, Shechtman, and
  Wang}]{zhang2018lpips}
Zhang R, Isola P, Efros AA, et~al. (2018) The unreasonable effectiveness of
  deep features as a perceptual metric. In: Proceedings of the {IEEE}
  Conference on Computer Vision and Pattern Recognition ({CVPR}), pp 586--595

\end{thebibliography}
\end{small}

\end{document}